\documentclass{article}
\usepackage{graphicx} 
\usepackage{amsmath,amssymb,amsthm,amsfonts}
\usepackage[margin=1in]{geometry}
\usepackage{lipsum}
\usepackage{authblk}
\usepackage{blindtext}
\usepackage{amsfonts}
\usepackage{graphicx}
\usepackage{epstopdf}
\usepackage{algorithmic}
\usepackage{xr-hyper}
\usepackage{hyperref}
\usepackage{float}

\newtheorem{theorem}{Theorem}
\newtheorem{proposition}{Proposition}
\newtheorem{lemma}{Lemma}
\newtheorem{corollary}{Corollary}
\newtheorem{definition}{Definition}
\newtheorem{assumption}{Assumption}

\newcommand{\R}{\mathbb{R}}
\newcommand{\E}{\mathbb{E}}

\newcommand{\by}{\boldsymbol{y}}
\newcommand{\bI}{\mathbf{I}}

\newcommand{\A}{\mathcal{A}}
\newcommand{\mx}{\mathcal{X}}
\newcommand{\my}{\mathcal{Y}}
\newcommand{\T}{\mathcal{T}}
\newcommand{\me}{\mathcal{E}}
\newcommand{\md}{\mathcal{D}}
\newcommand{\trace}{\textrm{Tr}}
\newcommand{\diag}{\textrm{diag}}
\newcommand{\op}{\textrm{op}}
\newcommand*\samethanks[1][\value{footnote}]{\footnotemark[#1]}
\newcommand\inp{\mathrel{\stackrel{\makebox[0.05pt]{\mbox{\normalfont \small P}}}{\rightarrow}}}

\title{In-Context Learning of Linear Systems: Generalization Theory and Application to Operator Learning}

\author{Frank Cole\thanks{School of Mathematics, University of Minnesota, Minneapolis, MN 55455. \texttt{\{cole0932, yulonglu, zhan7594\}@umn.edu}}, \; Yulong Lu\samethanks, \; Wuzhe Xu\thanks{Department of Mathematics and Statistics, University of Massachusetts Amherst, MA 01003. \texttt{wuzhexu@umass.edu}}, \; and \;Tianhao Zhang\samethanks}

\begin{document}

\maketitle

\begin{abstract}
    We study theoretical guarantees for solving linear systems in-context using a linear transformer architecture. For in-domain generalization, we provide neural scaling laws that bound the generalization error in terms of the number of tasks and sizes of samples used in training and inference. For out-of-domain generalization, we find that the behavior of trained transformers under task distribution shifts depends crucially on the distribution of the tasks seen during training. We introduce a novel notion of task diversity and show that it defines a necessary and sufficient condition for pre-trained transformers generalize under task distribution shifts. We also explore applications of learning linear systems in-context, such as to in-context operator learning for PDEs. Finally, we provide some numerical experiments to validate the established theory. 
\end{abstract}

\section{Introduction}
Transformers \cite{vaswani2017attention} have achieved tremendous success on natural language processing tasks, exemplified by large language models such as ChatGPT \cite{achiam2023gpt}. More recently, transformers have also found applications to computer vision \cite{khan2022transformers,liu2021swin}, physical sciences \cite{mccabe2023multiple, subramanian2024towards, ye2024pdeformer}, and other fields. Transformers, distinguished from feedforward neural networks by their self-attention mechanism, are designed to operate on large sequences of vectors. Remarkably, transformers exhibit the ability to perform \textit{in-context learning (ICL)}: pre-trained models can adapt to downstream tasks \textit{without updating their weights} by conditioning on a sequence of labeled examples.

A recent line of work aims to understand the mechanisms through which transformers perform in-context learning by studying their behavior on mathematically tractable problems. Specifically, several recent papers \cite{zhang2023trained,ahn2023transformers,mahankali2023one,kim2024transformers,bai2024transformers} have studied the behavior of transformers in solving regression problems in-context. In this setting, the input of the transformer is a sequence of the form
\begin{equation}\label{eqn:icllinearegression}
S = (x_1, f(x_1), \dots, x_n, f(x_n), x_{n+1}) \end{equation}
where $\{x_i\}_{i=1}^{n+1} \in \R^d$ and $f : \R^d \rightarrow \R$ are drawn from respective marginal distributions, and the model is trained to predict the unseen label $f(x_{n+1})$ by minimizing an $\ell^2$-loss averaged over the distributions of $x_1, \dots, x_{n+1}$ and $f$. We will refer to each function $f$ as a \textit{task} and each vector $x_i \in \R^d$ as a \textit{covariate}. At test time, the model is given a sequence of the form in Equation \eqref{eqn:icllinearegression} -- corresponding to a new task which the model has not seen during training -- and asked to predict the label of $x_{n+1}.$ Several works have studied this problem when the tasks are defined by linear functions $f(x) = \langle w, x \rangle.$ Specifically, the work \cite{zhang2023trained} proved that the transformer which minimizes the $\ell^2$ population loss provably predicts the correct label in the large-sample limit by deriving quantitative bounds on the prediction error in terms of the sequence length $n$. Remarkably, this bound holds for any task $w^{\ast}$, \textit{even those far beyond the support of the distribution of tasks seen during training}. 

While most of the existing literature focuses on learning scalar-valued functions in-context, many scientific applications involve learning a vector-valued function. In the setting of in-context learning, studying vector-valued functions significantly increases the complexity of the task space, even for linear functions. It is therefore unclear whether the same guarantees for in-context learning scalar-valued functions hold in the multivariate setting. 

\subsection{Our contributions}
In this work, we derive mathematical guarantees for solving linear systems of the form $Ax = y$, where $x,y \in \R^d$ and $A \in \R^{d \times d},$ in-context. Our main contributions are outlined below.
\begin{enumerate}
    \item \textbf{(In-domain generalization):} We prove an in-domain generalization bound for the expected risk of pre-trained linear transformers for solving linear systems in-context. Our bound depends on the number of labels per task during training and testing, as well as the number of tasks seen during pre-training; see Theorem \ref{thm:indomaingen} for a precise statement of our results. 
    \item \textbf{(Out-of-domain generalization):} We find that when transformers are trained to solve linear systems in-context, their adaptability under task distribution shifts depends significantly on the \textit{diversity} of the distribution of tasks seen during training. We propose a novel notion of task diversity which guarantees pre-trained transformers to well generalize under task distribution shifts. Further, we provide several positive (resp. counter-) examples of distributions that fulfill (resp. do not fulfill) the task diversity condition; see Theorems \ref{thm: taskshift} and \ref{thm: diversitysufficient} and Propositions \ref{thm: taskshiftimproved} and \ref{thm: oodgennecessary} for precise statements.
    \item \textbf{(Application to operator learning of PDEs):} We show how our results can be applied to prove guarantees for the in-context learning of operators between infinite-dimensional spaces. We prove a general result (Theorem \ref{thm: indomaingenoperator}) which translates our in-domain generalization bounds for in-context learning linear systems to the operator learning setting, which we then use to derive generalization bounds for in-context learning solution operators to linear elliptic PDEs; see Corollary \ref{cor: generalizationellipticpde}.
\end{enumerate}

\subsection{Related work}
\paragraph{ICL of statistical models and theoretical analysis} The works \cite{garg2022can,akyurek2022learning,von2023transformers} show by explicit construction  that linear transformers can implement a single step of gradient descent for ICL of linear regression. The works \cite{zhang2023trained, mahankali2023one, ahn2023transformers} concurrently studied the approximation and statistical guarantees of single-layer linear transformers trained to perform in-context linear regression. Various works refined the analysis of linear transformers for linear regression tasks, extending the theory to affine functions \cite{zhang2024context}, robustness properties \cite{anwar2024adversarial}, and  asymptotic behavior \cite{lu2024asymptotic}. Going beyond linear functions, several works \cite{bai2024transformers,guo2023transformers, wu2024transformers} have leveraged algorithm unrolling to prove approximation error bounds for nonlinear transformers in representing more complex functions. The work \cite{li2024one} studied the in-context learning ability of one-layer transformers with softmax attention and interpreted the inference-time prediction as a one-nearest neighbor estimator. The works \cite{kim2024transformers, mroueh2023towards} studied in-context learning from the perspective of non-parametric statistics. Optimization analysis of in-context learning has recently been studied in various settings, including for learning nonlinear functions \cite{huang2023context, kim2024transformers2, oko2024pretrained}, using nonlinear attention \cite{yang2024context}, and using multiple heads \cite{chen2024training}. The work \cite{wu2023many} analyzes the performance of SGD for in-context learning linear regression under finitely-many pre-training tasks.

Several recent works have also studied the behavior of in-context learning under non-IID observation models. In particular, the works \cite{sander2024transformers, zheng2024mesa, cole2025context} study the ability of linear transformers to represent linear dynamical systems in-context, deriving connections and differences between the IID and non-IID settings. The work \cite{li2023transformers} studies the generalization of transformers in learning more general state space models, while the work \cite{goel2024can} investigates the ability of transformers to approximate Kalman filtering. Finally, another line of work \cite{nichani2024transformers,edelman2024evolution,chen2024unveiling} studies the ability of softmax transformers to in-context learn Markov chains over finite state spaces. In this setting, it is shown that transformers learn to implement an induction head mechanism \cite{olsson2022context}.

Among the aforementioned works, the settings of \cite{zhang2023trained, ahn2023transformers,chen2024training} are closest to us. Our theoretical bound on the population risk extends the results of \cite{zhang2023trained, ahn2023transformers} for the  linear regression tasks to the tasks of solving linear systems. Different from those works where the data and task distributions are assumed to be Gaussian, the task distributions considered here are  fairly general and non-Gaussian. Our out-of-domain generalization error bounds substantially improve the prior generalization bound \cite{mroueh2023towards} established for general ICL problems. In particular, we show that the error due to the distribution shift can be reduced by a factor of $1/m$, where $m$ is the prompt-length of a downstream task. This rigorously justifies that the key robust feature of pre-trained transformer model with respect to task distribution shifts, which has previously been empirically observed in ICL of PDEs (see e.g. \cite{yang2023context,yang2024pde}), but has only been rigorously studied by \cite{zhang2023trained} in the linear regression setting. 

\paragraph{Foundation models and ICL for PDEs} Several transformer-based foundation models for solving PDEs have been developed in \cite{subramanian2024towards,mccabe2023multiple,ye2024pdeformer,sun2024towards} where the pre-trained transformers are adapted to downstream tasks with  fine-tuning on additional datasets. The works \cite{yang2023context,yang2024pde} study the in-context operator learning of differential equations where the adaption of the pre-trained model is achieved by only conditioning on new prompts. While these empirical work show great adaptivity of scientific foundation models for solving PDEs, their theoretical guarantees are largely open. To the best of our knowledge, this work is the first to derive the theoretical error bounds of transformers for learning linear PDEs in context.

\section{Problem set-up}\label{sec:setup}

\subsection{In-context learning of linear systems}\label{subsec:iclsystems}
Let $\A \subset \R^{d \times d}$ be a collection of invertible matrices, let $P_{\A}$ denote a probability measure on $\A,$ and let $P_x$ denote a probability measure on $\R^d.$ Our goal is to train a single model to solve the (possibly infinite) family of linear systems
$$ \{Ay = x: x,y \in \R^d, \; A \in \A\}
$$
in-context. More precisely, at the training stage, we are given a training dataset comprising $N$ length-$n$ prompts of input-output pairs $\{(x_i,y_i)\}_{i=1}^{n}$ with $x_i \sim P_x,$ $y_i = A^{-1} x_i,$ and $A \sim P_{\A}.$ We refer to $x \in \R^d$ as a \textit{covariate} and $y \in \R^d$ as a \textit{label}. An ICL model, pre-trained on the data above, is asked to predict the label $y$ corresponding to a new covariate $x$ and a new task $A$, conditioned on new prompt $\{(y_i,x_i)\}_{i=1}^{m},$ which may or may not have the same distribution as the training prompts. Further, the length $m$ of the prompt at test time may differ from the length $n$ of the prompts during pre-training. We emphasize that an in-context learning model is fundamentally different from a supervised learning model: the latter is trained to solve a single task, while the former is trained to solve an entire family of tasks indexed by $\mathcal{A}$. To make the analysis tractable, we make the following assumptions on the task space $\mathcal{A}$ and the marginal distribution $P_x$ of the covariates.

\begin{assumption}\label{assumption:taskanddatadistr}
    There exist constants $c_{\mathcal{A}}, C_{\mathcal{A}} > 0$ such that $\|A^{-1}\|_{\textrm{op}} \leq c_{\mathcal{A}}$ and $\|A\|_{\textrm{op}} \leq C_{\mathcal{A}}$ for all $A \in \mathcal{A}.$ In addition, the covariate distribution $P_x$ is Gaussian $\mathcal{N}(0,\Sigma).$
\end{assumption}

\subsection{Linear transformer architecture}\label{subsec:tf}
Let $k, $ be arbitrary, let $P,Q \in \R^{k \times k},$ and let $\theta = (P,Q).$ A single-layer linear transformer parameterized by $\theta$ is a mapping that acts on $\R^k$-valued sequences of arbitrary length. Specifically, if $Z = \begin{pmatrix}
    z_1 & \dots & z_T
\end{pmatrix} \in \R^{k \times T}$ is a length-$T$ sequence in $\R^k$, then $\textrm{TF}_{\theta}(Z) \in \R^{k \times T}$ is a length-$T$ sequence in $\R^k$ defined by
\begin{equation}\label{eqn:lineartf}
    \textrm{TF}_{\theta}(Z) = Z + PZ \cdot \frac{Z^T Q Z}{\rho(T)},
\end{equation}
where $\rho(T) \in \mathbb{N}$ is a normalization factor depending on the sequence length. A linear transformer defines a model for in-context learning the family of linear systems defined in Section \ref{subsec:iclsystems} as follows. Given a prompt $\{(x_i,y_i)\}_{i=1}^{N}$ with $\{x_i\}_{i=1}^{N} \sim P_x,$ $y_i = A^{-1} x_i,$ and $A \in \mathcal{A},$ and an unlabeled covariate $x_{n+1},$ define the embedding matrix
\begin{equation}\label{eqn: embeddingmatrix} Z_{\{(x_i,y_i)\}_{i=1}^{n}, x_{n+1}} = Z = \begin{pmatrix}
    x_1 & \dots & x_n & x_{n+1} \\
    y_1 & \dots & y_n & 0
\end{pmatrix} \in \R^{2d \times (n+1)}.
\end{equation}
For $P,Q \in \R^{2d \times 2d}$, define $\textrm{TF}_{\theta}(Z)$ as in Equation \eqref{eqn:lineartf}, where the normalization is defined as $\rho(T) = T-1.$ We then define the prediction $y_{n+1}^{\theta}$ for $A^{-1}x_{n+1}$ corresponding to $\theta$ by
\begin{equation}\label{eqn: tfprediction}
    y^{\theta}_{n+1} = (\textrm{TF}_{\theta}(Z))_{(d+1):2d, (n+1)} \in \R^d,
\end{equation}
i.e., $y_{n+1}^{\theta}$ is the vector defined by the final $d$ rows of the final column of the matrix $\textrm{TF}_{\theta}(Z).$ An analogous architecture was used for in-context learning linear regression models in \cite{zhang2023trained}. It can be shown that $y_{n+1}^{\theta}$ depends only on a subset of the parameters of $\theta$; in particular, if $P = \begin{pmatrix}
    P_{11} & P_{12} \\ P_{21} & P_{22}
\end{pmatrix}$ and $Q = \begin{pmatrix} Q_{11} & Q_{12} \\ Q_{21} & Q_{22} \end{pmatrix}$ with $P_{ij},Q_{ij} \in \R^{d \times d}$, then
\begin{equation}
    y_{n+1}^{\theta} = \begin{pmatrix}
        P_{21} & P_{22}
    \end{pmatrix} \cdot \frac{ZZ^T}{n} \cdot \begin{pmatrix}
        Q_{11} \\ Q_{21}
    \end{pmatrix} \cdot x_{n+1}.
\end{equation}
In addition, we observed empirically that, for linear transformers trained to solve linear systems in-context, all of the blocks of $P$ and $Q$ except $P_{22}$ and $Q_{11}$ were almost identically zero, which corroborates results from \cite{zhang2023trained}. This motivates us to focus on the parameterization $P = \begin{pmatrix}
    0 &0 \\ 0 & P_{22}
\end{pmatrix}$ and $Q = \begin{pmatrix} Q_{11} & 0 \\ 0 & 0 \end{pmatrix}$ in our subsequent analysis. We abuse notation and henceforth write $P = P_{22} \in \R^{d \times d}$ and $Q = Q_{11} \in \R^{d \times d}$. Under this parameterization, the prediction $y_{n+1}^{\theta}$ simplifies to
\begin{equation}
    y_{n+1}^{\theta} = P \left( \frac{1}{n} \sum_{i=1}^{n} y_i x_i^T \right) Q x_{n+1}.
\end{equation}

\subsection{Generalization of ICL}\label{subsec: genoficl}
Our goal is to find parameters $\theta =(P,Q)$ that minimize the population risk functional
\begin{equation}\label{eqn: poprisk}
    \mathcal{R}_n(\theta) = \E_{x_1, \dots x_{n+1} \sim P_x, A \sim P_{\A}}\left[\left\|y_{n+1}^{\theta} - A^{-1} x_{n+1} \right\|^2 \right].
\end{equation}
Since we only have access to finite samples from $P_x$ and $P_{\A},$ the parameters $\theta$ are trained by minimizing the corresponding empirical risk functional. Precisely, for $N \in \mathbb{N}$ and $1 \leq i \leq N$, sample $x_{1,i}, \dots, x_{n,i}, x_{n+1,i} \sim P_x$ and $A_1, \dots, A_N \sim P_{\A}$, define $y_{j,i} = A_i^{-1} x_{j,i}$ for $1 \leq j \leq n$, and define $y_{n+1,i}^{\theta}$ according to Equation \eqref{eqn: tfprediction}. Then, define the empirical risk $\mathcal{R}_{n,N}$ by
\begin{equation}\label{eqn: emprisk}
    \mathcal{R}_{n,N}(\theta) = \frac{1}{N} \sum_{i=1}^{N} \left\| y_{n+1,i}^{\theta}- A_i^{-1} x_{n+1,i}  \right\|^2.
\end{equation}
At inference time, a pre-trained transformer with parameters $\widehat{\theta} \in \textrm{arg}\min_{\theta} \mathcal{R}_{n,N}(\theta)$ is expected to make predictions on a downstream task consisting of a new prompt $\{(x_i,y_i)\}_{i=1}^{m} = \{(x_i,A^{-1}x_i)\}_{i=1}^{m}$ for a new $A \in \mathcal{A}$ and a new query $x_{n+1}.$ Since the covariates $x_1, \dots, x_m$ and the task $A$ at inference time may or may not follow the same distribution as those seen during pre-training, it is natural to consider two different scenarios. First, we define the \textbf{in-domain generalization error} by
\begin{equation}\label{eqn: indomainrerror}
    \mathcal{R}_m(\widehat{\theta}) = \E_{x_1, \dots x_{m+1} \sim P_x, A \sim P_{\A}}\left[\left\|y_{m+1}^{\widehat{\theta}} - A^{-1} x_{m+1} \right\|^2 \right]
\end{equation}
defined by the $m$-sample population risk. On the other hand, if $P_x'$ and $P_{\A}'$ are new distributions describing the covariates and task at inference time, we define the \textbf{out-of domain (OOD) generalization error} by
\begin{equation}\label{eqn: ooderror}
    \mathcal{R}_m'(\widehat{\theta}) = \E_{x_1, \dots x_{m+1} \sim P_x', A \sim P_{\A}'}\left[\left\|y_{m+1}^{\widehat{\theta}} - A^{-1} x_{m+1} \right\|^2 \right].
\end{equation}

\section{Main results}\label{sec:results}

\subsection{In-domain generalization}\label{subsec:indomain}
Our first result is a bound on the in-domain generalization error in terms of the number of pre-training tasks $N$, the number of examples per task $n$ during pre-training, and the number of examples per task $m$ at inference. In order to fully characterize the convergence rate with respect to $N$, we introduce a few definitions. First, if $P$ is a probability measure and $\mathcal{F} \subset L^2(P)$ is a compact subset, then we let $\mathcal{N}(\mathcal{F})$ denote the set of functions $g$ for which
$ \min_{f \in \mathcal{F}} \|f-g\|_{L^2(P)}^2
$
is attained by more than one element of $\mathcal{F}$. In other words, $\mathcal{N}(\mathcal{F})$ is the set of functions which have more than one best approximation within $\mathcal{F}$.

\begin{definition} \label{def:fast}
    Let $P_{\A}$ and $P_x$ denote the task and covariate distributions respectively. Let $P$ denote the joint law of $(A^{-1}, x_1, \dots, x_{n+1})$ with $A \sim P_{\A}$ and $(x_1, \dots, x_{n+1}) \sim P_x^{\otimes (n+1)}$. 
    For $\theta = (P,Q),$ $P, Q \in \R^{d \times d}$ define $F_{\theta} \in L^2(P)$ by
    $$ F_{\theta}(A^{-1}, y_1, \dots, y_{n+1}) = P A^{-1} X_n Q x_{n+1}, \; X_n = \frac{1}{n} \sum_{i=1}^{n} x_i x_i^T,
    $$
    and for $M > 0$, let $\mathcal{F}_{\theta,M} = \{F_{\theta}: \|P\|, \|Q\| \leq M\}.$ We say that the triple $(P_{\A},P_x, \mathcal{F}_{\theta,M})$ satisfies the \textbf{structural condition} if the function $F^{\ast} \in L^2(P)$ defined by
    $ F^{\ast}(A^{-1},x_1, \dots, x_{n+1}) = A^{-1}x_{n+1}
    $
    satisfies
    \begin{equation} \label{eq:fstar}
        F^{\ast} \notin \overline{\mathcal{N}(\mathcal{F}_{\theta,M})},
    \end{equation}
    where $\overline{\mathcal{F}}$ denotes the $L^2(P)$-closure.
\end{definition}
As we will see, the fast rate condition  is sufficient to ensure that the in-domain generalization decays at the parametric rate as a function of the number of tasks $N$. We are now ready to state our main result on in-domain generalization. Recall that, under Assumption \ref{assumption:taskanddatadistr}, the task space $\A$ is contained in the set $\{A: \|A\|_{\textrm{op}} \leq C_A, \; \|A^{-1}\|_{\textrm{op}} \leq c_{\A}\}$, and the covariate distribution $P_x = \mathcal{N}(0,\Sigma)$ is Gaussian.

\begin{theorem}\label{thm:indomaingen}
    Adopt Assumption \ref{assumption:taskanddatadistr}. Let $\widehat{\theta} = \textrm{arg}\min_{\|\theta\| \leq M} \mathcal{R}_{n,N}(\theta)$, where the norm is defined by $\|\theta\| = \max\{\|P\|_{\textrm{op}},\|Q\|_{\textrm{op}}\}$ and $M > 0$ is any positive number such that $M \geq \max(1, \|\Sigma^{-1}\|_{\textrm{op}}).$ Assume in addition that $m \leq n.$ Then 
    \begin{equation}
    \mathcal{R}_m(\widehat{\theta}) \lesssim \frac{1}{m} + \frac{1}{n^2} + \frac{1}{\sqrt{N}}, \; \textrm{with probability $\geq 1-\frac{1}{\textrm{poly}(N)}$}.
    \end{equation}
    If in addition Equation \eqref{eq:fstar} holds, then \begin{equation}
    \mathcal{R}_m(\widehat{\theta}) \lesssim \frac{1}{m} + \frac{1}{n^2} + \frac{1}{N}, \; \textrm{with probability $\geq 1-\frac{1}{\textrm{poly}(N)}$}
    \end{equation}
    where the implicit constants hidden in ``$\lesssim$" depend (polynomially) on $d$, $c_{\A}$, $C_{\A}$, and the norms of $\Sigma$ and $\Sigma^{-1}$.
    
\end{theorem}

Theorem \ref{thm:indomaingen} demonstrates that the in-domain generalization error tends to zero as the various sample sizes tend to $\infty.$ Theorem \ref{thm:indomaingen} will also be an important tool when we prove OOD generalization error bounds in the next section. The assumption that $m \leq n$ can be lifted at the expense of a slower rate with respect to $n$; see a more general version of Theorem \ref{thm:indomaingen} in the supplementary material.

Some discussion of the dependence of the estimate in Theorem \ref{thm:indomaingen} on the various sample sizes is in order. In the language of \cite{kim2024transformers}, the term $O\left(\frac{1}{m} + \frac{1}{n^2} \right)$ is the \textit{in-context generalization error}, while the term $O \left( \frac{1}{\sqrt{N}} \right)$ is the \textit{pre-training generalization error.} The in-context generalization error arises from constructing an explicit transformer to achieve this rate. Our construction is similar to the attention weights studied for in-context linear regression in \cite{mahankali2023one, ahn2023linear, zhang2023trained}. We highlight that the length of the prompts during training and inference play different roles in the overall error rate; a similar phenomenon was observed in Theorem 4.2 of \cite{zhang2023trained}, and we explain this in more detail in the proof sketch in Section \ref{sec: pfsketches}.

The  rate w.r.t $N$ in generalization error depends on whether or not the fast rate condition holds. The upshot of the fast rate condition is that it ensures that the excess loss satisfies a Bernstein condition (see, e.g., Definition 1.2 in \cite{mendelson2008obtaining}) which quickly implies the $O(1/N)$ rate. Without the fast rate condition, the rate $O(1/\sqrt{N})$ is optimal in the worst case. The fast rate condition would immediately hold if the function class $\mathcal{F}_{\theta,M}$ was convex, but this is not the case for transformer neural networks. While it is difficult to verify whether the fast rate condition holds for general distributions $P_A$ and $P_x$, we conjecture that the condition holds when the pre-training task distribution $P_{\A}$ is sufficiently diverse, in the sense of Definition \ref{def:diverse}. We discuss this conjecture in more detail in the supplementary material. In addition, our numerical experiments indicate that the $O(1/N)$ holds in practice. We leave it as a future exercise to verify the fast rate condition rigorously.

\subsection{Out-of-domain generalization}\label{subsec:ood}
While Theorem \ref{thm:indomaingen} provides quantitative bounds on the in-domain generalization error, there is no guarantee that the data defining the inference-time task will follow the same distribution of the data used during pre-training. Even within the support of the pre-training task distribution $P_{\A},$ Theorem \ref{thm:indomaingen} only bounds the expected prediction error, which is weaker than a pointwise bound. In this section, we bound the OOD generalization error under appropriate assumptions on the pre-training distribution. Recall that $\mathcal{R}_n$ denotes the population risk defined by the distributions $P_x$ and $P_{\A}$ and $\mathcal{R}_n'$ denotes the population risk defined by the distributions $P_x'$ and $P_{\A}'$. We say that the pre-trained transformer with parameters $\widehat{\theta}$ \textbf{achieves OOD generalization} if the OOD generalization error $\mathcal{R}_m'(\widehat{\theta})$ converges to zero in probability as $m,n,N \rightarrow \infty$.

\paragraph{Task shifts}  We first consider the task-shift case where $P_x = P_x'$ and $P_{\A} \neq P_{\A}'.$ If the pre-training task distribution $P_{\A}$ is too narrow, we cannot expect to achieve OOD generalization. However, as we will prove, if $P_{\A}$ is diverse relative to $P_{\A'}$ in an appropriate sense, then the OOD generalization error tends to zero as the amount of data increases. In our analysis, it will be useful to consider the functionals 
\begin{align*}
    \mathcal{R}_{\infty}(\theta) &= \E_{A \sim P_{\A}} \left[\left\|(PA^{-1} \Sigma Q - A^{-1}) \Sigma^{1/2} \right\|_F^2 \right], \\
    \mathcal{R}_{\infty}'(\theta) &= \E_{A \sim P_{\A}'} \left[\left\|(P(A')^{-1} \Sigma Q - (A^{-1})') \Sigma^{1/2} \right\|_F^2 \right].
\end{align*}
An easy calculation shows that the risk functionals $\mathcal{R}_n$ and $\mathcal{R}_n'$ converge uniformly on compact sets to $\mathcal{R}_{\infty}$ and $\mathcal{R}_{\infty}'$ respectively with an error of $O \left( \frac{1}{n} \right).$ We denote by $\mathcal{M}_{\infty}'$ the set of minimizers of $\mathcal{R}_{\infty}'.$ The functionals $\mathcal{R}_{\infty}$ and $\mathcal{R}_{\infty}'$ and their minimizers are crucial to our notion of diversity.

\begin{definition}\label{def:diverse}
    We say that the distribution $P_{\A}$ is \textbf{diverse} relative to the distribution $P_{\A}'$ if $\mathcal{M}_{\infty} \subseteq \mathcal{M}_{\infty}'.$ In other words, whenever $\theta$ is optimal for the distribution $P_{\A}$ (in the large sample limit), it is also optimal for $P_{\A}'$.
\end{definition}
Diversity has been observed as a crucial property of pre-training datasets for the generalization and robustness of transformers (\cite{raventos2023pretraining,kim2024task}) but, to the best of our knowledge, a rigorous notion of diversity for in-context learning has yet to be introduced. The following result demonstrates the utility of our definition of diversity: the diversity of $P_{\A}$ is a sufficient condition to achieve OOD generalization. 

\begin{theorem}\label{thm: taskshift}
    Let $\widehat{\theta} \in \textrm{arg}\min_{\|\theta\| \leq M} \mathcal{R}_{n,N}(\theta)$ for $M > 0$ sufficiently large. Suppose that the distribution $P_{\A}$ is diverse relative to the distribution $P_{\A}'$. Then the OOD generalization error satisfies:
    \begin{equation}
        \mathcal{R}_m'(\widehat{\theta}) \lesssim \mathcal{R}_m(\widehat{\theta}) +  \textrm{dist}(\widehat{\theta},\mathcal{M}_{\infty}) + \frac{d(P_{\A},P_{\A}')}{m},
    \end{equation}
    where $d(P_{\A},P_{\A}')$ is a measure of distance between the distributions $P_{\A}$ and $P_{\A'}$ (defined precisely in the supplementary material).
\end{theorem}
Theorem \ref{thm: taskshift} decomposes the OOD generalization error into three terms. The first term is the in-domain generalization error, which converges to $0$ with high-probability as $m,n,N \rightarrow \infty$ by Theorem \ref{thm:indomaingen}. The second term is the distance between the empirical risk minimizer and the set of minimizers of $\mathcal{R}_{\infty}$. In the supplementary material, we show that this term tends to $0$ in probability. The final term captures the distance between the distributions $P_{\A}$ and $P_{\A}'$ but, notably, it inherits a factor of $\frac{1}{m}.$ Thus, Theorem \ref{thm: taskshift} demonstrates that the distribution shift error tends to zero as the sample size $m \rightarrow \infty.$ Consequently, Theorem \ref{thm: taskshift} demonstrates the empirical risk minimizer $\widehat{\theta}$ achieves OOD generalization, provided the distribution $P_{\A}$ of tasks during pre-training is diverse relative to the distribution $P_{\A}'$ of tasks at inference time.

Having identified diversity as a sufficient condition to achieve OOD generalization, we would like to better understand when this condition holds. While it may be difficult to verify the condition in practice, we provide two sufficient conditions in the next result. To state the theorem we recall that the centralizer of a subset $\mathcal{A} \subseteq \R^{d \times d}$ is the set $\mathcal{C}(\A) := \{M \in \R^{d \times d}: MA = AM \forall A \in \A\}.$

\begin{theorem}\label{thm: diversitysufficient}
    Let $P_{\A}$ and $P_{\A}'$ denote two distributions satisfying Assumption \ref{assumption:taskanddatadistr}.
    \begin{enumerate}
        \item If $\textrm{supp}(P_{\A}') \subseteq \textrm{supp}(P_{\A}),$ then $P_{\A}$ is diverse relative to $P_{\A}'.$
        \item Define $\mathcal{S}(P_{\A}) := \{A_1A_2^{-1}: A_1, A_2 \in \textrm{supp}(P_{\A})\}.$ If $\mathcal{C}(\mathcal{S}(P_{\A})) = \{c \mathbf{I}_d: c \in \R\},$ then $P_{\A}$ is diverse relative to $P_{\A}'.$
    \end{enumerate}
\end{theorem}
The first statement of Theorem \ref{thm: diversitysufficient} is natural: it says that the pre-training task distribution is diverse whenever the downstream task distribution is a “subset” of it, in the
sense of supports. The second condition is particularly interesting because it implies OOD generalization (by Theorem \ref{thm: taskshift}) regardless of the inference-time distribution $P_{\A}'.$ The second condition based on the centralizer of the set $\mathcal{S}(P_{\A})$ (henceforth referred to as the centralizer condition) is less obvious, but heuristically it enforces that the support of $P_{\A}$ must be large enough that the only matrices which can commute with all elements in $\mathcal{S}(P_{\A})$ are multiples of the identity matrix. Under the centralizer condition, we can state achieve a stronger OOD generalization result, with bounds that hold uniformly over the support of the downstream task distribution.

\begin{corollary}\label{thm: taskshiftimproved}
    Let $\A \subset \R^{d \times d}$ satisfy Assumption \ref{assumption:taskanddatadistr} and let $P_{\A}$ and $P_{\A}'$ be two probability measures on $\A$. Let $\mathcal{S}(P_{\A}) = \{A_1A_2^{-1}: A_1, A_2 \in \textrm{supp}(P_{\A})\}$ and suppose that the centralizer of $\mathcal{S}(P_{\A})$ consists only of multiples of the identity matrix. Then the transformer with parameters $\widehat{\theta} \in \textrm{arg}\min_{\|\theta\| \leq M} \mathcal{R}_n(\theta)$ satisfies
    \begin{equation} \lim_{m,n,N \rightarrow \infty} \sup_{A \in \textrm{supp}(P_{\A}')}  \E_{x_1, \dots, x_{m+1} \sim \mathcal{N}(0,\Sigma)} \left[ \left\| y_{m+1}^{\widehat{\theta}} - A^{-1} x_{m+1} \right\|^2 \right] = 0 \; \textrm{in probability.}
    \end{equation}
\end{corollary}
Corollary \ref{thm: taskshiftimproved} implies that the prediction error of the transformer pre-trained on tasks sampled from $P_{\A}$, measured \textit{uniformly} over the support of $P_{\A}'$, converges to zero in probability as the amount of data increases. This is quite surprising, as the transformer parameters are not trained to minimize the uniform error. In addition, the result is independent of the downstream task distribution, as long as Assumption \ref{assumption:taskanddatadistr} is satisfied. The key assumption on $\mathcal{P}(\A)$ is the centralizer condition.

Finally, it is natural to ask whether the diversity condition on $P_{\A}$ is necessary to achieve OOD generalization. For example, in \cite{zhang2023trained}, it is shown for in-context learning of linear regression (i.e., the in-context learning of one-dimensional linear maps $x \mapsto \langle w, x \rangle)$ that the minimizer of the $\mathcal{R}_n$ achieves OOD generalization under very general assumptions. In the next result, we show by example that, in the worst case, the diversity of $P_{\A}$ is necessary to achieve OOD generalization. To this end, recall that a set $S \subset \R^{d \times d}$ of real symmetric matrices is \textbf{simultaneously diagonalized} by an orthogonal matrix $U$ if for every $A \in S$, $U^T A U$ is diagonal.

\begin{proposition}\label{thm: oodgennecessary}
    Let $U$ be a $d \times d$ orthogonal matrix, and suppose $P_{\A}$ and $P_{\A}'$ are two task distributions such that $\textrm{supp}(P_{\A})$ is simultaneously diagonalized by $U$ and $\textrm{supp}(P_{\A}')$ is not simultaneously diagonalized by $U$. Then, under additional mild conditions on the two distributions, $P_{\A}$ is not diverse relative to $P_{\A}'$. Consequently, there exist transformer parameters $\theta$ such that $\mathcal{R}_{\infty}(\theta) = 0$ but $\mathcal{R}_{\infty}'(\theta) = \Omega(1).$
\end{proposition}
Proposition \ref{thm: oodgennecessary} uncovers a fundamental difference between in-context learning scalar-valued and vector-valued linear models. For scalar-valued functions, the diversity of the pre-training data is completely determined by the number of samples. For vector-valued functions, the appropriate measure of diversity depends not only on the amount of data, but also on the structure of the tasks used for pre-training. The key property of the distribution $P_{\A}$ constructed in Proposition \ref{thm: oodgennecessary} its simultaneous diagonalizability. For such distributions, there are several minimizers of $\mathcal{R}_{\infty}$ which 'memorize' the diagonalization, and thus are not robust under task distribution shifts.

\paragraph{Covariate shifts} We now turn our attention to the OOD generalization error under shifts in the covariate distribution. Specifically, let $P_x = \mathcal{N}(0,\Sigma)$ and $P_{\A}$ denote the covariate and task distributions during pre-training, and let $P_x' = \mathcal{N}(0,\Sigma')$ and $P_{\A}' = P_A$ denote the covariate and task distributions at inference time. In related settings, it has been observed both theoretically and empirically (\cite{zhang2023trained}) that the linear transformer architecture is not robust under general covariate distribution shifts. Below, we prove a stability estimate on the OOD generalization error under covariate shifts.

\begin{theorem}\label{thm: covariateshift}
    Let $P_x = \mathcal{N}(0,\Sigma)$ and $P_x' = \mathcal{N}(0,\Sigma')$ denote the covariate distributions during pre-training and inference and let $P_{\A}$ denote the task distribution. Let $\Sigma = W \Lambda W^T$ and $\Sigma' = U \Lambda' U^T$ be eigendecompositions of the covariance matrices $\Sigma$ and $\Sigma'$, where $\Lambda, \Lambda'$ are diagonal and $W,U$ are orthogonal. Let $\widehat{\theta} \in \textrm{arg}\min_{\|\theta\| \leq M} \mathcal{R}_n(\theta)$. Then the OOD generalization error under covariate shifts satisfies:
    \begin{equation}
        \E_{x_1, \dots, x_{m+1} \sim \mathcal{N}(0,\Sigma'), A \sim P_{\A}} \left[ \left\| y_{m+1}^{\widehat{\theta}} - A^{-1} x_{m+1} \right\|^2 \right] \lesssim \mathcal{R}_m(\widehat{\theta}) + \|\Sigma - \Sigma'\|_{\textrm{op}} + \frac{1}{m} \|U-W\|_{\textrm{op}}.
    \end{equation}
\end{theorem}
Unlike the case of task distribution shifts, Theorem \ref{thm: covariateshift} does not imply that the OOD generalization error under covariate shifts decreases to $0$ by taking the amount of data to $\infty.$

\section{Application to operator learning}\label{sec:icon}
Let $\mx$ and $\my$ be infinite-dimensional vectors spaces endowed with norms $\| \cdot \|_{\mx}$ and $\| \cdot \|_{\my}$ respectively, and let $\mathcal{L}(\mx,\my)$ denote the set of bounded linear operators from $\mx$ to $\my$. In many data-driven problems in science and engineering, the goal is to learn an operator $T \in \mathcal{L}(\mx,\my)$, for appropriate choices of $\mx$ and $\my$, from a finite collection of input-output pairs. In order to make this problem computationally tractable, one must approximate the infinite-dimensional operator $T$ by a mapping between finite-dimensional spaces. To this end, let $\me: \mx \rightarrow \R^d$ be an encoder and let $\md: \R^d \rightarrow \my$ be a decoder and, for $T \in \mathcal{L}(\mx,\my),$ define the finite-dimensional mapping $T_{\me,\md}: \R^d \rightarrow \R^d$ defined by $T_{\me,\md} := \me \circ T \circ \md$. In practice, the encoder and decoder can be chosen from prior knowledge of the problem or learned directly from data. Throughout this section, we will assume that $\me$ and $\md$ are linear mappings.
It follows that $T_{\me,\md}$ is a linear map for any $T \in \mathcal{L}(\mx,\my).$ Given a collection of operators $\mathcal{T} \subset \mathcal{L}(\mx,\my)$, we aim to in-context learn the family of infinite-dimensional linear systems defined by
\begin{equation}
    \{T(f) = u: f \in \mx, u \in \my, T \in \mathcal{T}\}.
\end{equation}
The observational set-up for in-context learning a family of operators is directly analogous to the set-up for in-context learning a family of linear systems described in Section \ref{subsec:iclsystems}.
In this case, the model is pre-trained to in-context learn the family of linear systems arising from the finite-dimensional approximation (as defined by $\me$ and $\md$) of the family of operators $\T$, but at inference time, we aim to predict the output of the infinite-dimensional operator. As such, we will incur an error not only due to finite sample sizes, but also due to the finite-dimensional discretization of $\T.$

The in-context learning of infinite-dimensional operators has attracted a lot of recent attention due to various developments of foundation models for solving scientific problems such as partial differential equations \cite{mccabe2023multiple,subramanian2024towards,sun2024towards}. Recent empirical studies on in-context operator learning \cite{yang2023context,yang2024pde} have demonstrated that transformers can efficiently in-context learn solution operators to various archetypal equations. In the remainder of this section, we demonstrate that our generalization error bounds for in-context learning linear systems can be readily translated into generalization error bounds for in-context operator learning. To the best of our knowledge, this is the first mathematical explanation of in-context operator learning.

Our first result is an abstract in-domain generalization error bound for in-context operator learning in terms of the in-domain generalization error for in-context learning the associated family of linear systems, as well as several quantities depending on the encoder and decoder.

\begin{theorem}\label{thm: indomaingenoperator}
    Let $\mx$, $\my$, $\T \subset \mathcal{L}(\mx,\my)$, $\me$, $\md$, $P_{\mx}$, and $P_{\T}$ be as defined in Section \ref{sec:icon}. Suppose that the set of finite-dimensional linear maps $\A := \{T_{\me,\md}: T \in \T\}$ and the pushforward distribution $P_x := \me_{\#} P_{\mx}$ satisfy Assumption \ref{assumption:taskanddatadistr}. Let $\epsilon_{\me,\md}$ and $C_{\md}$ be constants such that
    \begin{align}
        \sup_{T \in \T} \left\|(T_{\me,\md} - T)(f) \right\|_{\my} &\leq \epsilon_{\me,\md} \|f\|_{\mx} \; \textrm{for all $f \in \mx$;} \\
        \|\md(x)\|_{\my} &\leq C_{\md} \|x\| \; \textrm{for all $x \in \R^d$.} 
    \end{align}
    Let $\widehat{\theta}$ denote the transformer parameters to minimize the empirical risk
    $$ \mathcal{R}_{n,N}(\theta) = \frac{1}{N} \sum_{i=1}^{N} \left\| y_{n+1,i}^{\theta}- T_{\me,\md}^{(i)} (\me(f_{n+1,i}))  \right\|^2,
    $$
    where $ \{T^{(i)}\} \sim P_T$ and $ \{f_{i,j}\} \sim P_{\mx}.$ Then   
    \begin{equation}
        \E_{f_1, \dots, f_{m+1} \sim P_{\mx}, T \sim P_{\T}} \left[\left\|\md\left(y_{m+1}^{\widehat{\theta}}\right) - T(f_{m+1}) \right\|_{\my}^2 \right] \lesssim \epsilon_{\me,\md}^2 + C_{\md}^2 \mathcal{R}_m(\widehat{\theta}),
    \end{equation}
    where $\mathcal{R}_m(\widehat{\theta})$ is the in-domain generalization error defined in Theorem \ref{thm:indomaingen}.
\end{theorem}
Theorem \ref{thm: indomaingenoperator} bounds the (in-domain) generalization of in-context operator learning by a sum of two terms. The first term represents the error incurred by using a finite-dimensional approximation of the operator class $\T$ during pre-training; this term depends only on the encoder and decoder and typically decreases as the embedding dimension $d$ tends to $\infty.$ The second term represents the statistical error of in-context learning the associated family of linear systems from a finite sample collection, which can be bounded by Theorem \ref{thm:indomaingen}. The statistical error inherits a factor of $C_{\md},$ which represents the fact that the norm $\| \cdot \|_{\my}$ may be stronger than the Euclidean norm on $\R^d$. Typically the factor $C_{\md}$ increases as the embedding dimension $d$ increases, so there is a natural trade-off between the two terms in Theorem \ref{thm: indomaingenoperator}. Note that Theorem \ref{thm: indomaingenoperator} can easily be applied to obtain bounds on the OOD generalization error under task shifts, assuming that the pre-training distribution on $\T$ satisfies the diversity condition described in Section \ref{subsec:ood}.

While Theorem \ref{thm: indomaingenoperator} is general enough to apply to a variety of practical situations, it is unclear how the constants $\epsilon_{\me,\md}$ and $C_{\md}$ typically behave. To illustrate the utility of $\ref{thm: indomaingenoperator},$ we apply it to a family of operators arising from a prototypical family of PDEs. In more detail, consider the linear elliptic PDE on a bounded Lipschitz domain $\Omega \subset \R^d$ 
\begin{equation}\label{eqn: ellipticpde}
    \begin{cases}
        -\nabla(a(x) \nabla u(x)) + V(x) u(x) = f(x), \; x \in \Omega \\
        u \equiv 0, \; x \in \partial \Omega,
    \end{cases}
\end{equation}
where $a \in L^{\infty}(\Omega)$ is strictly positive, $V \in L^{\infty}(\Omega)$ is nonnegative and bounded, and $f \in L^2(\Omega).$ By elliptic regularity, the solution operator mapping the source $f$ to the solution $u$ can be viewed as a bounded linear operator from $\mx = L^2(\Omega)$ to $\my = H^1(\Omega).$ To define the encoder and decoder, we let $\{\phi_k\}_{k=1}^{\infty}$ denote a basis for $L^2(\Omega)$ and define $\me(f) = \textrm{Proj}\left( \textrm{span} \left(\{\phi_k\}_{k=1}^{d} \right) \right)$ and $\md(x) = \sum_{k=1}^{d} x_k \phi_k$ for some $d \in \mathbb{N}.$ In the subsequent analysis, we will take $\{\phi_k\}_{k=1}^{\infty}$ to be a $P^1$-finite element (\cite{brenner2007mathematical}), i.e., a basis of piecewise linear functions. With Theorem \ref{thm: indomaingenoperator} in hand, we can derive generalization error bounds for in-context learning of solution operators to PDEs of the form \eqref{eqn: ellipticpde}. For simplicity, we present the explicit bound in the corollary below for the one-dimensional elliptic problem \eqref{eqn: ellipticpde} with $\Omega = [0,1]$, whose proof can be found in the supplementary material. We remark that the result can be extended to multi-dimensions by  incorporating a suitably adapted discretization error bound.

\begin{corollary}\label{cor: generalizationellipticpde}
Consider the elliptic problem  \eqref{eqn: ellipticpde} on $\Omega  = [0,1]$.  
Adopt the setting of Theorem \ref{thm: indomaingenoperator}, where $\mx = L^2(\Omega)$, $\my = H^1(\Omega)$, $\me$ and $\md$ are defined according to the above paragraph, and $\mathcal{T} \subset \mathcal{L}(\mx,\my)$ is a collection of solution operators to the PDE \eqref{eqn: ellipticpde} where $a \in L^{\infty}(\Omega)$ is strictly positive and $V \in L^{\infty}(\Omega)$. Assume that the distribution $P_{\mx}$ is a centered Gaussian measure on $\mx$ with trace class covariance operator $\Sigma_{\mx}.$ Then there exist constants $\alpha_1(\Sigma_{\mx}), \alpha_2(\Sigma_{\mx}) > 0,$ depending only on $\Sigma_{\mx}$ and the number of basis functions $d$, such that
    \begin{equation}
            \E_{f_1, \dots, f_{m+1} \sim P_{\mx}, T \sim P_{\T}} \left[\left\|\md\left(y_{m+1}^{\widehat{\theta}}\right) - T(f_{m+1}) \right\|_{H^1(\Omega)}^2 \right] \lesssim \frac{1}{d^2} + \frac{d^2}{m} + \frac{\alpha_1(\Sigma_{\mx})}{n^2} + \frac{\alpha_2(\Sigma_{\mx})}{\sqrt{N}}.
    \end{equation}
    Under the structural condition of Definition \ref{def:fast}, the bound above improves to
    \begin{equation}
        \E_{f_1, \dots, f_{m+1} \sim P_{\mx}, T \sim P_{\T}} \left[\left\|\md\left(y_{m+1}^{\widehat{\theta}}\right) - T(f_{m+1}) \right\|_{H^1(\Omega)}^2 \right] \lesssim \frac{1}{d^2} + \frac{d^2}{m} + \frac{\alpha_1(\Sigma_{\mx})}{n^2} + \frac{\alpha_2(\Sigma_{\mx})}{N}.
    \end{equation}
\end{corollary}

\section{Proofs of main results}\label{sec: pfsketches}
We provide the proofs of Theorems \ref{thm:indomaingen}, \ref{thm: taskshift}, \ref{thm: diversitysufficient} and \ref{thm: indomaingenoperator} -- our main  contributions -- and defer the proofs of the other results and auxiliary lemmas to the supplementary material. We begin with Theorem \ref{thm:indomaingen}. Due to the page limit, we provide a sketch of the proof and defer the full argument to the supplementary material.

\paragraph{Proof sketch of Theorem \ref{thm:indomaingen}} We decompose the population risk $\mathcal{R}_m(\widehat{\theta})$ of the empirical risk minimizer $\widehat{\theta}$ into three main sources of error:
\begin{enumerate}
    \item \textbf{Truncation error:} Most of our subsequent estimates depend on the size of the support of the data distribution, but the Gaussian covariates have noncompact support. To mitigate this, we show that the problem of bounding the population risk $\mathcal{R}_m(\widehat{\theta})$ can be reduced to that of bounding a 'truncated' version of the population risk, where the data is restricted to a compact set. Due to the fast tails of the Gaussian distribution, we show that the error incurred at this step is mild.
    \item \textbf{Statistical error:} The statistical error is measured by $(\mathcal{R}_m(\widehat{\theta}) - \mathcal{R}_m(\theta^{\ast}))$, where $\theta^{\ast}$ is a minimizer of $\mathcal{R}_m$. This term is bounded using tools from empirical process theory, and the bounds depend on the Rademacher complexity of the class of linear transformers. When the fast rate condition is satisfied, the statistical error is bounded by the fixed point of the local Rademacher complexity, which is $O(1/N)$, while without the fast rate condition we rely on a more brutal estimate that incurs an $O(1/\sqrt{N})$ error.
    \item \textbf{Approximation error:} The approximation error is measured by $\mathcal{R}_m(\theta^{\ast})$, where $\theta^{\ast}$ is a minimizer of $\mathcal{R}_n$ over the transformer neural network class (recall that the transformer parameters are trained on prompts of length $n$ and tested on prompts of length $m$). While it is difficult to compute the population risk minimizer exactly, we use a similar approach to \cite{zhang2023trained} to bound this term. Specifically, one can show that for $\theta = (P,Q),$ the transformer prediction given the length-$n$ prompt is
    $$ y^{\theta}_{n+1} = P \cdot A^{-1} \cdot \left(\frac{1}{n} \sum_{i=1}^{n} x_i x_i^T \right) \cdot Q \cdot x_{n+1}.
    $$
    When $n$ is large, the empirical covariance $\frac{1}{n} \sum_{i=1}^{n} x_i x_i^T$ is close to the population covariance $\Sigma$. Thus, a natural parameterization is $P \approx \mathbf{I}_d$ and $Q \approx \Sigma^{-1}$. We prove that when $P = \mathbf{I}_d$, the optimal $Q$ is a $O(n^{-1})$-perturbation of $\Sigma^{-1}$. Plugging these optimal parameters into the $m$-sample population risk $\mathcal{R}_m$ yields an error of $O(m^{-1} + n^{-2}).$
\end{enumerate}
Combining the three sources of error and carefully tracking the constants gives the bound as it is stated in Theorem \ref{thm:indomaingen}. We now give the full proofs of Theorems \ref{thm: taskshift} and \ref{thm: diversitysufficient}.

\begin{proof}[Proof of Theorem \ref{thm: taskshift}]
We recall that $\widehat{\theta} \in \textrm{argmin}_{\|\theta\| \leq M} \mathcal{R}_{n,N}(\theta)$ is the ERM. To prove Theorem \ref{thm: taskshift}, we will prove the more general bound
\begin{equation}\label{eqn: generalbound}
    \mathcal{R}_m'(\widehat{\theta}) \lesssim \mathcal{R}_m(\widehat{\theta}) + \frac{d(P_{\A},P_{\A}')}{m} + \textrm{dist}(\widehat{\theta},\mathcal{M}_{\infty}(P_{\A}))^2 + \textrm{dist}(\widehat{\theta},\mathcal{M}_{\infty}(P_{\A}'))^2,
\end{equation}
which holds even without the diversity condition. When $P_{\A}$ is diverse relative to $P_{\A}'$, Equation \ref{eqn: generalbound} implies Theorem \ref{thm: taskshift}, because in this case the inequality 
$$ \textrm{dist}(\widehat{\theta},\mathcal{M}_{\infty}(P_{\A}'))^2 \leq \textrm{dist}(\widehat{\theta},\mathcal{M}_{\infty}(P_{\A}))^2
$$
holds. To prove Equation \eqref{eqn: generalbound}, let $\theta_{\ast} = (P_{\ast},Q_{\ast})$ denote a projection of $\widehat{\theta}$ onto the set $\mathcal{M}_{\infty}$ of minimizers of $\mathcal{R}_{\infty}$; similarly, let $\theta_{\ast}' = (P_{\ast}',Q_{\ast}')$ denote a projection of $\widehat{\theta}$ onto $\mathcal{M}_{\infty}'$. (Note that such projections exist by compactness of the parameter space, but they need not be unique. This will not play a role in the proof.) We decompose the OOD error as
\begin{align*}
    \mathcal{R}_m'(\widehat{\theta}) = \mathcal{R}_m(\widehat{\theta}) + \left( \mathcal{R}_m'(\widehat{\theta}) - \mathcal{R}_m'(\theta_{\ast}') \right) + \left(\mathcal{R}_m(\theta_{\ast}') - \mathcal{R}_m(\theta^{\ast}) \right) +  \left( \mathcal{R}_m(\theta_{\ast}) - \mathcal{R}_m(\widehat{\theta}) \right).
\end{align*}

    Taking the infimum over all projections $\theta_{\ast}$ and $\theta_{\ast}'$ of $\widehat{\theta}$ onto $\mathcal{M}_{\infty}(P_{\A})$ and $\mathcal{M}_{\infty}(P_{\A}')$, followed by the supremum over $\widehat{\theta}$ in $\{\|\theta\| \leq M\}$, we arrive at the bound
    \begin{align*}
        \mathcal{R}_m'(\widehat{\theta}) &\leq \mathcal{R}_m(\widehat{\theta}) + \sup_{\|\widehat{\theta}\| \leq M} \inf_{\theta_{\ast}, \theta_{\ast}'} \left|\mathcal{R}_m(\theta_{\ast}) - \mathcal{R}_m'(\theta_{\ast}') \right| + \sup_{\|\theta_1\|, \|\theta_2\| \leq M, \|\theta_1 - \theta_2\| \leq \epsilon_2} \left| \mathcal{R}_m(\theta_1) - \mathcal{R}_{m}(\theta_2) \right| \\ &+ \sup_{\|\theta_1\|, \|\theta_2\| \leq M, \|\theta_1 - \theta_2\| \leq \epsilon_1} \left| \mathcal{R}_m'(\theta_1) - \mathcal{R}_{m}'(\theta_2) \right|.
    \end{align*}
    The second and third terms can be bounded using a simple Lipschitz continuity estimate. Note that for $m$ sufficiently large and $\theta = (P,Q)$ with $\|\theta\| \leq M$, we have
    $$ \|(PA^{-1}X_mQ - A^{-1})\Sigma^{1/2}\|_{F}^2 \lesssim c_{\A}^2(1+\|\Sigma\|_{\textrm{op}} M^2)^2 \trace(\Sigma)
    $$
    for any $A \in \textrm{supp}(P_{\A}).$ It follows that 
    $$ R_m(\theta) = \E_{A \sim P_{\A}, X_m} [\|(PA^{-1}X_mQ - A^{-1})\Sigma^{1/2}\|_{F}^2]
    $$
    is $O\left( c_{\A}^2(1+\|\Sigma\|_{\textrm{op}} M^2)^2 \trace(\Sigma) \right)$-Lipschitz on $\{\|\theta\| \leq M\}.$ We therefore have
    $$ \sup_{\|\theta_1\|, \|\theta_2\| \leq M, \|\theta_1 - \theta_2\| \leq \epsilon_1} \left| \mathcal{R}_m(\theta_1) - \mathcal{R}_{m}(\theta_2) \right| \lesssim \left( c_{\A}^2(1+\|\Sigma\|_{\textrm{op}} M^2)^2 \trace(\Sigma) \right) \epsilon_1^2.
    $$
    An analogous bound holds for $ \sup_{\|\theta_1\|, \|\theta_2\| \leq M, \|\theta_1 - \theta_2\| \leq \epsilon_2} \left| \mathcal{R}_m'(\theta_1) - \mathcal{R}_{m}'(\theta_2) \right|$, since the test distribution $P_{\A}'$  also satisfies Assumption \ref{assumption:taskanddatadistr}. To bound the term $\left|\mathcal{R}_m(\theta_{\ast}) - \mathcal{R}_m'(\theta_{\ast}') \right|$, we use Lemma \ref{lem: popriskexpression} in the supplementary material to show that for any $\theta = (P,Q)$,
    $$ \mathcal{R}_m(\theta) = \mathcal{R}_{\infty}(\theta) + \frac{1}{m} \E_{A \sim P_{\A}} \left[\trace(PA^{-1} \Sigma Q \Sigma Q^T \Sigma A^{-1} P^T) + \trace_{\Sigma}(Q \Sigma Q^T) \trace(P A^{-1} \Sigma A^{-1} P^T) \right]
    $$
    and
    $$ \begin{aligned}
        \mathcal{R}_m'(\theta)  & = \mathcal{R}_{\infty}'(\theta) + \frac{1}{m}\E_{A \sim P_{\A}'} \Big[\trace(P(A')^{-1} \Sigma Q \Sigma Q^T \Sigma (A')^{-1} P^T) \\
        & \qquad + \trace_{\Sigma}(Q \Sigma Q^T) \trace(P (A')^{-1} \Sigma (A')^{-1} P^T) \Big].
    \end{aligned} 
    $$
    In particular, since $\theta_{\ast} \in \textrm{argmin}_{\theta} \mathcal{R}_{\infty}(\theta)$ and $\theta_{\ast}' \in \textrm{argmin}_{\theta} \mathcal{R}_{\infty}'(\theta)$, and each functional achieves 0 as its minimum value, we have
    \begin{align*}
        &\left|\mathcal{R}_m(\theta_{\ast}) - \mathcal{R}_m'(\theta_{\ast}') \right| \leq \frac{1}{m} \Big| \E_{A \sim P_{\A}} \big[\trace(P_{\ast}A^{-1} \Sigma Q_{\ast} \Sigma Q_{\ast}^T \Sigma A^{-1} P_{\ast}^T)\\
        & \qquad + \trace_{\Sigma}(Q_{\ast} \Sigma Q_{\ast}^T) \trace(P_{\ast} A^{-1} \Sigma A^{-1} P_{\ast}^T) \big] -\E_{A \sim P_{\A}'} \big[\trace(P_{\ast}'(A')^{-1} \Sigma Q_{\ast}' \Sigma (Q_{\ast}')^T \Sigma (A')^{-1} (P_{\ast}')^T) \\ &
        \qquad  + \trace_{\Sigma}(Q_{\ast}' \Sigma (Q_{\ast}')^T) \trace(P_{\ast}' (A')^{-1} \Sigma (A')^{-1} (P_{\ast}')^T) \big] \Big| \\ &=: \frac{1}{m} \left| \E_{A \sim P_{\A}}[f(A;\theta_{\ast})] - \E_{A' \sim P_{\A}'}[f(A';\theta_{\ast}')] \right|.
    \end{align*}
    It follows that
    \begin{align}\label{eq: distancedef} \sup_{\|\widehat{\theta}\| \leq M} \inf_{\theta_{\ast}, \theta_{\ast}'} \left|\mathcal{R}_m(\theta_{\ast}) - \mathcal{R}_m'(\theta_{\ast}') \right| &\leq \frac{1}{m}\sup_{\|\widehat{\theta}\| \leq M} \inf_{\theta_{\ast}, \theta_{\ast}'} \left| \E_{A \sim P_{\A}}[f(A;\theta_{\ast})] - \E_{A' \sim P_{\A}'}[f(A';\theta_{\ast}')] \right| \\
    &=: \frac{1}{m} d(P_{\A},P_{\A}'),
    \end{align}
    where the infimum is taken over all $\theta_{\ast} \in \textrm{argmin}_{\theta \in \mathcal{M}_{\infty}(P_{\A})} \|\theta - \widehat{\theta}\|^2$ and $\theta_{\ast}' \in \textrm{argmin}_{\theta' \in \mathcal{M}_{\infty}(P_{\A}')} \|\theta' - \widehat{\theta}\|^2.$ Combining the estimates for each individual term in the error decomposition, we obtain the final bound in the statement of Theorem \ref{thm: taskshift}.
    
    The fact that the bound stated in Theorem \ref{thm: taskshift} tends to zero as the sample size $(m,n,N) \rightarrow \infty$ follows from examination of each term in the estimate: the in-domain generalization error $\mathcal{R}_m(\widehat{\theta})$ tends to zero in probability by Theorem \ref{thm:indomaingen}, the term $\frac{d(P_{\A},P_{\A}')}{m}$ is deterministic and tends to zero as $m \rightarrow \infty$, and $\textrm{dist}(\widehat{\theta}, \mathcal{M}_{\infty})$ tends to zero as $N$ and $n$ tend to infinity, respectively, by Proposition \ref{convergenceofminimizers}.
\end{proof}

We now turn to Theorem \ref{thm: diversitysufficient}. The key step to prove Theorem \ref{thm: diversitysufficient} is to characterize the set of minimizers of the limiting risk $\mathcal{R}_{\infty}$, which warrants its own result.

\begin{proposition}\label{minimizercharacterization}
    Fix a task distribution $P_{\A}$ satisfying Assumption \ref{assumption:taskanddatadistr}. Then $\theta = (P,Q)$ is a minimizer of $\mathcal{R}_{\infty}$ if and only if $P$ commutes with all elements of the set $\{A_1 A_2^{-1}: A_1, A_2 \in \textrm{supp}(P_{\A})\}$ and $Q$ is given by $Q = \Sigma^{-1} A_0 P^{-1} A_0^{-1}$ for any $A_0 \in \textrm{supp}(P_{\A}).$
\end{proposition}
Assuming Proposition \ref{minimizercharacterization}, we now present the short proof of Theorem \ref{thm: diversitysufficient}.

\begin{proof}[Proof of Theorem \ref{thm: diversitysufficient}]
    Item 1) of Theorem \ref{thm: diversitysufficient} is a direct consequence of Proposition \ref{minimizercharacterization}. To prove item 2), let $\theta_{\ast} = (P_{\ast},Q_{\ast})$ be a minimizer of $\mathcal{R}_{\infty}$. Then Proposition \ref{minimizercharacterization} implies that $P_{\ast} \in \mathcal{C}(\mathcal{S}(P_{\A})).$ Since the centralizer of $\mathcal{S}(P_{\A})$ is trivial by assumption, this implies that $P_{\ast} = c \mathbf{I}_d$ for some $c \in \R \backslash \{0\}.$ Using the characterization of minimizers of $\mathcal{R}_{\infty}$ derived in Proposition \ref{minimizercharacterization}, we have that $Q_{\ast}$ solves the equation $c A^{-1} \Sigma Q_{\ast} = A^{-1}$ for all $A \in \textrm{supp}(P_{\A})$, and therefore $Q = c^{-1} \Sigma^{-1}.$
\end{proof}

We now prove Proposition \ref{minimizercharacterization}.
\begin{proof}[Proof of Proposition \ref{minimizercharacterization}]
    Recall that
    $$ \mathcal{R}_{\infty}(\theta) = \E_{A \sim P_{\A}}[\|(P A^{-1} \Sigma Q - A^{-1})\Sigma^{1/2}\|_F^2], \; \theta = (P,Q),
    $$
    and $\mathcal{M}_{\infty}(P_{\A}) = \textrm{argmin}_{\theta} \mathcal{R}_{\infty}(\theta).$ Let us first prove that for any $P_{\A}$ satisfying Assumption \ref{assumption:taskanddatadistr}, $\theta \in \mathcal{M}_{\infty}(P_{\A})$ if and only if $PA^{-1} \Sigma Q = A^{-1}$ for all $A \in \textrm{supp}(P_{\A}).$ Let us first observe that the minimum value of $\mathcal{R}_{\infty}$ is 0 - this is attained, for instance, at $P = \mathbf{I}_d$ and $Q = \Sigma^{-1}.$ It is clear that if the equality $PA^{-1} \Sigma Q = A^{-1}$ holds over the support of $P_{\A}$, then $\E_{A \sim P_{\A}}[\|(PA^{-1}\Sigma Q - A^{-1})\Sigma^{1/2}\|_F^2] = 0$. Conversely, suppose $(P,Q)$ satisfies $\E_{A \sim P_{\A}}[\|(PA^{-1}\Sigma Q - A^{-1})\Sigma^{1/2}\|_F^2] = 0.$
    Fix $A_0 \in \textrm{supp}(P_{\A})$. If $A_0$ is an atom of $P_{\A}$ then the desired claim is clear. Otherwise, let $\epsilon > 0$ be small enough that the $\epsilon$-ball around $A_0$ is contained in $\textrm{supp}(P_{\A})$, and let $p_{A,\epsilon}(A_0)$ denote the normalized restriction of $P_{\A}$ to the ball of radius $\epsilon$ centered about $A_0$. Then the equality $\E_{A \sim P_{\A}}[\|(PA^{-1}\Sigma Q - A^{-1})\Sigma^{1/2}\|_F^2] = 0$ implies that $$\E_{A \sim p_{A,\epsilon}(A_0)}[\|(PA^{-1}\Sigma Q - A^{-1})\Sigma^{1/2}\|_F^2] = 0$$ for each $\epsilon > 0$. Since $p_{A,\epsilon}(A_0)$ converges weakly to the Dirac measure centered at $A_0$, we have that $\|(PA_0^{-1} \Sigma Q - A_0^{-1})\Sigma^{1/2}\|_F^2 = 0$, and hence that $P A_0^{-1} \Sigma Q = A_0^{-1}.$ As $A_0$ was arbitrary, this concludes the first part of the proof.

    Now, suppose $\theta = (P,Q)$ is a minimizer of $\mathcal{R}_{\infty}.$ By the previous argument, this is equivalent to the system of equations $PA^{-1} \Sigma Q = A^{-1}$ holding simultaneously for all $A \in \textrm{supp}(P_{\A}).$ In particular, for any fixed $A_0 \in \textrm{supp}(P_{\A})$, the equation $P A_0^{-1} \Sigma Q = A_0^{-1}$ can be solved for $ Q$, yielding $Q = \Sigma^{-1} A_0 P^{-1} A_0^{-1}.$ Since the matrix $Q$ is constant, this implies that the function $A \mapsto A P^{-1} A^{-1}$ is a constant on the support of $P_{\A}.$ We have therefore shown that the minimizers of $\mathcal{R}_{\infty}$ can be completely characterized as $\{(P, \Sigma^{-1} A_0 P^{-1} A_0^{-1}): P \in \R^{d \times d}\},$ where $A_0$ is any element of $\textrm{supp}(P_{\A})$. In addition, the fact that the function $A \mapsto A P^{-1} A^{-1}$ is constant on the support of $P_{\A}$ implies that $P$ commutes with all products of the form $\{A_1 A_2^{-1}: A_1, A_2 \in \textrm{supp}(P_{\A})\}$.
\end{proof}

Finally, we present the proof of Theorem \ref{thm: indomaingenoperator}.

\begin{proof}[Proof of Theorem \ref{thm: indomaingenoperator}]
    By an application of the triangle inequality, we have
    \begin{align*}
        &\E_{f_1, \dots, f_{m+1} \sim P_{\mx}, T \sim P_{\T}} \left[\left\|\md\left(y_{m+1}^{\widehat{\theta}}\right) - T(f_{m+1}) \right\|_{\my}^2 \right] \\ 
        &\leq 2 \left(\E \left[ \left\|T_{\mathcal{E},\mathcal{D}}(f_{m+1}) - T(f_{m+1}) \right\|_{\my}^2 + \left\|\md\left(y_{m+1}^{\widehat{\theta}}\right) - T_{\mathcal{E},\mathcal{D}}(f_{m+1}) \right\|_{\my}^2  +  \right] \right).
    \end{align*}
    The first term represents the discretization error due to projecting the inputs onto a finite-dimensional subspace, while the second term represents the statistical error of learning the finite-dimensional representations. For the discretization error, we have
    $$ \E \left[ \left\|T_{\mathcal{E},\mathcal{D}}(f_{m+1}) - T(f_{m+1}) \right\|_{\my}^2 \right] \leq \epsilon_{\mathcal{E},\mathcal{D}}^2
    $$
    by assumption. For the statistical error, we have
    \begin{align*}
      \E \left[ \left\|\md\left(y_{m+1}^{\widehat{\theta}}\right) - T_{\mathcal{E},\mathcal{D}}(f_{m+1}) \right\|_{\my}^2 \right] &= \E \left[ \left\|\md\left(y_{m+1}^{\widehat{\theta}}  - (T \circ \mathcal{E})(f_{m+1}) \right) \right\|_{\my}^2 \right] \\
      &\leq C_{\mathcal{D}}^2 \E \left[ \left\|\md\left(y_{m+1}^{\widehat{\theta}}  - (T \circ \mathcal{E})(f_{m+1}) \right) \right\|^2 \right],
    \end{align*}
    where the norm in the second line is the Euclidean norm. The expectation on the second line is simply the in-domain generalization error for learning the family of matrices $\{T \circ \mathcal{E}: T \in \mathcal{T}\}$ in context, given by $\mathcal{R}_m(\widehat{\theta})$. We conclude the proof.
\end{proof}

\section{Numerical results}
In this section, we numerically evaluate the in-context learning performance of linear transformers for linear systems. In \ref{subsec:RM}, we consider linear systems defined by random matrices, and in \ref{subsec:Elliptic}, we extend our study to the in-context operator learning for linear systems obtained via Galerkin discretizations of the linear elliptic PDE \eqref{eqn: ellipticpde}. We measure the generalization error with either $\ell^2$ error in the random matrix setting or $H^1$-error in the PDE setting. To assess the convergence rate of the generalization error, we use the {\em shifted relative error}, defined as the ratio between the difference of the test error and the population error (estimated using a sufficiently large amount of data) and the ground truth. The code for the numerical experiments is available at: \url{https://github.com/LuGroupUMN/ICL_Linear_Systems}.


\subsection{Random matrix}\label{subsec:RM}
Consider the linear system $Ay=x$, where the matrix $A \in \mathbb{R}^{d\times d}$ is defined as $A = Q D Q^T$. Here $Q$ is a random orthonormal matrix obtained from the QR decomposition of a random matrix. We denote by $\mathbf{U}_{d}[a,b]$ the distribution of $d \times d$ diagonal matrices with whose entries are independently sampled from the uniform distribution $U[a,b]$, that is $D \sim \mathbf{U}_{d}[a,b]$ as $D= \diag(\lambda_1,\dots,\lambda_{d}), ~
\lambda_i \overset{\mathrm{iid}}{\sim} U[a,b]$. The vector $y$ is sampled from a multivariate normal distribution $\mathcal{N}(0, \Sigma_d(\rho))$, where $\Sigma_d(\rho)$ denotes an equal-correlated covariance matrix, which is defined as $\Sigma_d(\rho) = (1-\rho)\bI_d + \rho F_d$, with $F_d\in \mathbb{R}^{d\times d}$ being a matrix of all ones.

\subsubsection{In-domain generalization}\label{subsec:in_domain_RM} We first investigate the scaling law stated in \ref{thm:indomaingen}. During training, we let $D \sim \mathbf{U}_d[1,2]$ and $\by \sim \mathcal{N}(0, \Sigma_d(0)) = \mathcal{N}(0, I_d)$. The numerical results presented in \ref{fig:scaling_RM} indicate scaling law of the in-domain generation error, measured by the shifted $\ell$-error, scales like $O(n
^{-2})$ and $O(m^{-1})$ with respect to the prompt sizes   $n$ and $m$ respectively, aligning well with  \ref{thm:indomaingen}.  \ref{fig:scaling_RM} also shows a fast rate $O(N^{-1})$ with respect to the task size $N$, suggesting that the structural condition  \ref{eq:fstar} is fulfilled, although its rigorous verification is still open. These numerical results collectively validate that, when training and test distributions coincide, the linear transformer can effectively learn from context, provided that the training prompt length $n$, the number of tasks, and the inference prompt length $m$ are sufficiently large.
\begin{figure}[htbp]
  \centering
  \includegraphics[width=\linewidth]{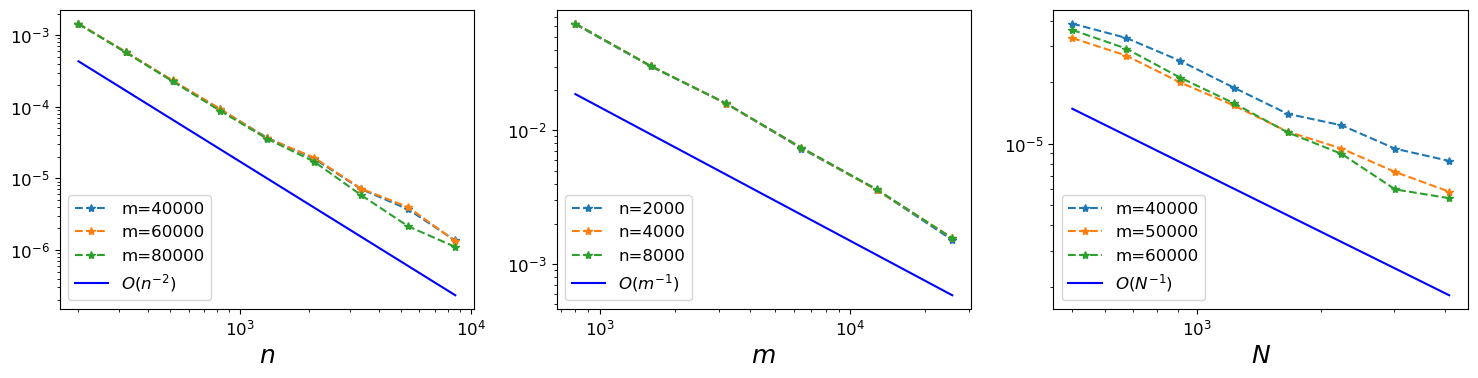}
  \caption{In-domain generalization with $d=10$. The left plot shows the relative $L^2$ error with respect to various training prompt length $n$ using $N=20000$ training tasks. The middle plot shows the relative $L^2$ error with respect to various inference prompt length $m$ using $N=20000$ training tasks. The right plot shows relative $L^2$ error with respect to various number of training tasks $N$ using a fixed training prompt length $n=10000$.}
  \label{fig:scaling_RM}
\end{figure}

\subsubsection{Out-of-domain generalization with diversity}\label{subsec:RM_diverse} We evaluate the out‑of‑domain generalization of the linear transformer when trained on a diverse task distribution (see \ref{def:diverse}). During training, we sample $D \sim \mathbf{U}_d[1,2]$ and $\by \sim \mathcal{N}(0, \Sigma_d(0)) = \mathcal{N}(0, \bI_d)$. When evaluating on downstream tasks, we assess performance under both task shifts and covariate shifts, see results in \ref{fig:RM_ood_diverse}. In the left plot, we vary the task distribution $\mathbf{U}_d[a,b]$ across various intervals $(a,b)$, and observe that the OOD generalization error curves (dash line) almost coincide with the in-domain generalization error curves (solid line) as the inference prompt length $m$ increases. This result aligns with \ref{thm: taskshift}, indicating the linear transformer generalizes well even under task distribution shifts. On the contrary, the right plot shows the error when the input undergoes a  covariate shift with various values of $\rho$ in the covariance $\Sigma(\rho)$, and suggests that the pretrained transformer fails to generalize in this case, which is consistent with the established bound in \ref{thm: covariateshift}.

\begin{figure}[H]
  \centering
  \includegraphics[width=\linewidth]{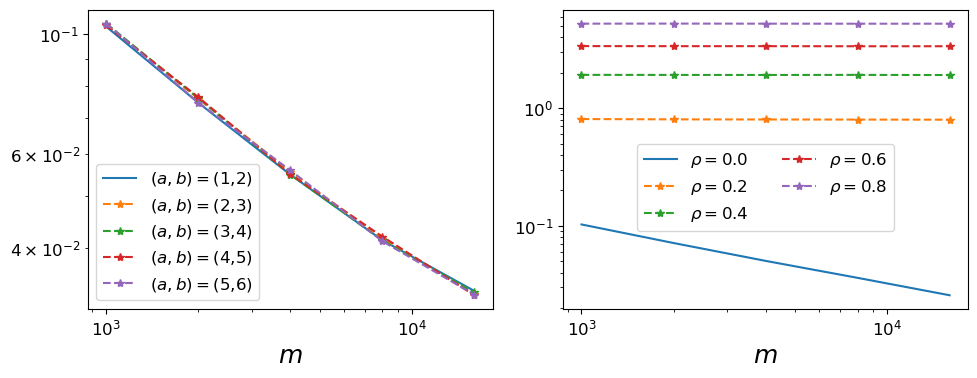}
  \caption{Out-of-domain generalization with $d=10$, $n=2000$ and $N=5000$. Training is performed with $D \sim \mathbf{U}_d[1, 2]$ and $y \sim \mathcal{N}(0, \Sigma_d(0)) = \mathcal{N}(0, \bI_d)$. Testing involves different settings of $\mathbf{U}_d[a, b]$ and $\mathcal{N}(0, \Sigma_d(\rho))$. The left plot shows the relative $L^2$ error under shifts in $V(x)$: the solid blue curve represents the training distribution, while dashed lines indicate shifted test distributions. The right plot presents analogous results for shifts in $f(x)$, where blue curve represents the training distribution and dashed lines indicate shifted test distributions.}
  \label{fig:RM_ood_diverse}
\end{figure}

\subsubsection{OOD generalization without diversity}\label{subsec: ood_not_diverse}  We also examine how the lack of task diversity during training impacts the OOD generalization. We train a linear transformer on task distribution of constant matrices, i.e. $D = c \bI_d, c \sim U[1,2]$ and $\by \sim \mathcal{N}(0, \bI_d)$, and then evaluate it on more diverse task distribution $\mathbf{U}_d[a,b]$ across various intervals $(a,b)$, and the results are presented in \ref{fig:RM_not_diverse}. Although the training loss admits infinitely many minimizers in the large-sample limit, namely any diagonal matrices pair $(K, K^{-1})$ with $K \sim \mathbf{U}_d[1, 2]$, only minimizers of the form $(c\bI_d, c^{-1}\bI_d)$, for $c \neq 0,$ generalize to more diverse downstream tasks. In the left plot of \ref{fig:RM_not_diverse}, a transformer initialized near $(\bI_d, \bI_d)$ generalizes to the downstream tasks, while the transformer initialized near $(K, K^{-1})$ (right plot) exhibits degraded performance. These results validate \ref{thm: oodgennecessary} and highlight the importance of task diversity in achieving OOD generalization. 

\begin{figure}
  \centering
  \includegraphics[width=.9\linewidth]{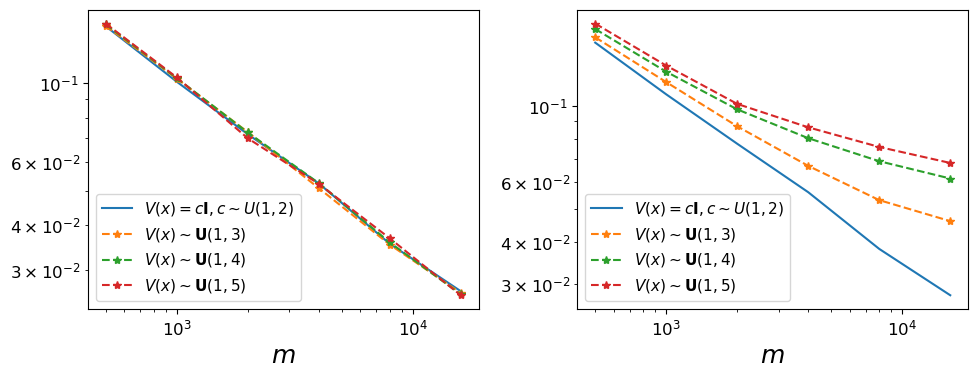}
  \caption{Diversity test with dimension $d=10$, training prompt length $n=2000$, number of tasks $N=5000$. Both plots show the relative $L^2$ error with respect to varying testing prompt length $m$. The left plot corresponds to a transformer initialized near the minimizer $(P, Q) = (\mathbf{I}_d, \mathbf{I}_d)$, while the right plot uses initialization $(P, Q) = (K, K^{-1})$, where $K$ is a diagonal matrix with entries sampled from the uniform distribution $U(1, 2)$. }
  \label{fig:RM_not_diverse}
\end{figure}

\subsection{Linear elliptic PDEs}\label{subsec:Elliptic}
In this section, we use a linear transformer to perform in-context operator learning for the linear elliptic PDE \eqref{eqn: ellipticpde}. We discretize the elliptic problem \ref{eqn: ellipticpde} via the Ritz–Galerkin method, yielding the finite-dimensional system $A u = f$, where we take both trial and test spaces as $\textrm{span} \left(\{\phi_k\}_{k=1}^{d} \right)$, with $\phi_k(x) = \sin (k \pi x),~k=1, \cdots, d$. Here the stiffness matrix $A$ results from projecting the differential operator $Lu = -\nabla(a(x) \nabla u(x)) + V(x) u(x)$ onto the sine basis, and $f$ collects the projections of source term $f(x)$. Below we evaluate the in-context operator learning capability of the linear transformer under this setting.

\subsubsection{In-domain generalization}
First, we numerically validate the scaling law stated in \ref{cor: generalizationellipticpde}. We fix $a(x)\equiv 0.1$, let $V(x) = e^{g(x)}, ~ g(x) \sim \mathcal{N}(0, 4(-\Delta + \alpha_1 I_d)^{-\beta_1})$ with $\alpha_1=\beta_1 = 2$, and let the source term $f(x)$ to be a white noise. As shown in \ref{fig:scaling_PDE}, the scaling law of the in-domain generation error with respect to $H^1$-norm are $O(n^{-2})$ and $O(m^{-1})$, which align with \ref{cor: generalizationellipticpde}. We observe a faster convergence rate in $N$, for the reason discussed in Section \ref{subsec:in_domain_RM}. These results validate that, when training and test distributions match, the linear transformer performs in-context operator learning for the linear elliptic PDE.

\begin{figure}[H]
  \centering
  \includegraphics[width=\linewidth]{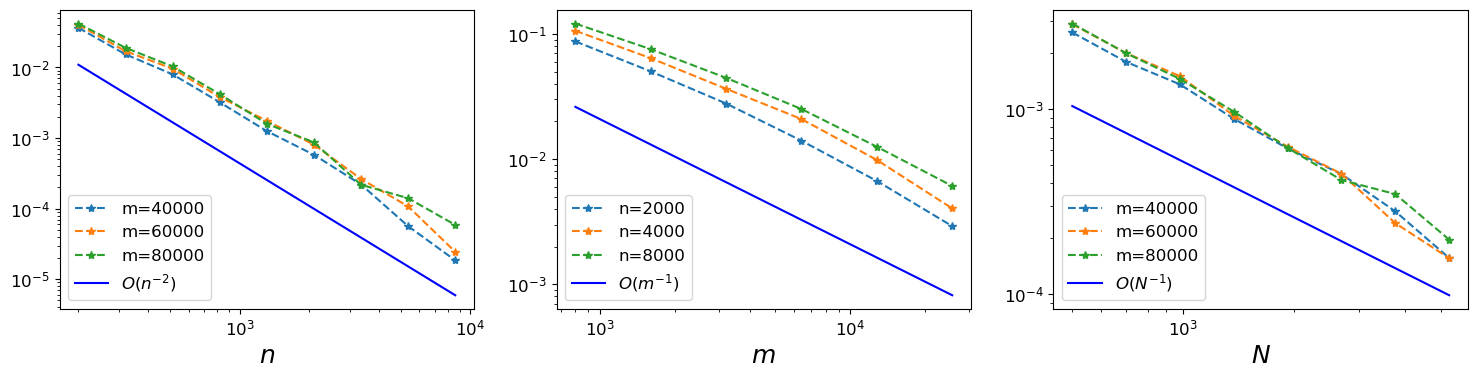}
  \caption{In-domain generalization test of elliptic PDE \eqref{eqn: ellipticpde} with $d=32$. The left plot shows the relative $H^1$ error with respect to various training prompt length $n$ using $N=20000$ training tasks. The middle plot shows the relative $H^1$ error with respect to various inference prompt length $m$ using $N=20000$ training tasks. The right plot shows relative $H^1$ error with respect to various number of training tasks $N$ using a fixed training prompt length $n=10000$.}
  \label{fig:scaling_PDE}
\end{figure}

\subsubsection{Out-of-domain generalization with diversity}\label{subsec:PDE_diverse} We evaluate the out‑of‑domain generalization of the linear transformer on PDE problem. We let $a(x) = 0.1 e^{h(x)}, ~ ~ h(x) \sim \mathcal{N}(0, (-\Delta + \alpha_1 I_d)^{-\beta_1})$; $V(x) = e^{g(x)}, ~ ~ g(x) \sim \mathcal{N}(0, 4(-\Delta + \alpha_2 I_d)^{-\beta_2})$ and $f(x) \sim \mathcal{N}(0, (-\Delta + \alpha_3 I_d)^{-\beta_3})$. All covariance hyperparameters are set to $\alpha_1=\beta_1=\alpha_2=\beta_2=\alpha_3=\beta_3=2$ during training. For inference, we assess performance under different task shifts and covariate shifts, and the results are presented in \ref{fig:PDE_ood_diverse}. For the task shift in $a(x)$ (first column), we vary $(\alpha_1,\beta_1)$ and for the task shift in $V(x)$ (second column), we vary $(\alpha_2,\beta_2)$. In both task shifts, the out-of-domain relative $H^1$ error curve (dash lines) decrease following the same rate of in-domain generalization error curve (solid line) as the inference prompt length increases, which is in agreement with \ref{thm: taskshift}, indicating that the linear transformer generalize under task shifts. In contrast, when introducing covariate shifts on $f(x)$ by varying $(\alpha_3,\beta_3)$, the the pretrained transformer fails to generalize, matching the result of \ref{thm: covariateshift}.

We also provide additional experiments on testing the OOD generalization when the training tasks are not diverse in \ref{sec:addnum}.

\begin{figure}[H]
  \centering
  \includegraphics[width=.8\linewidth]{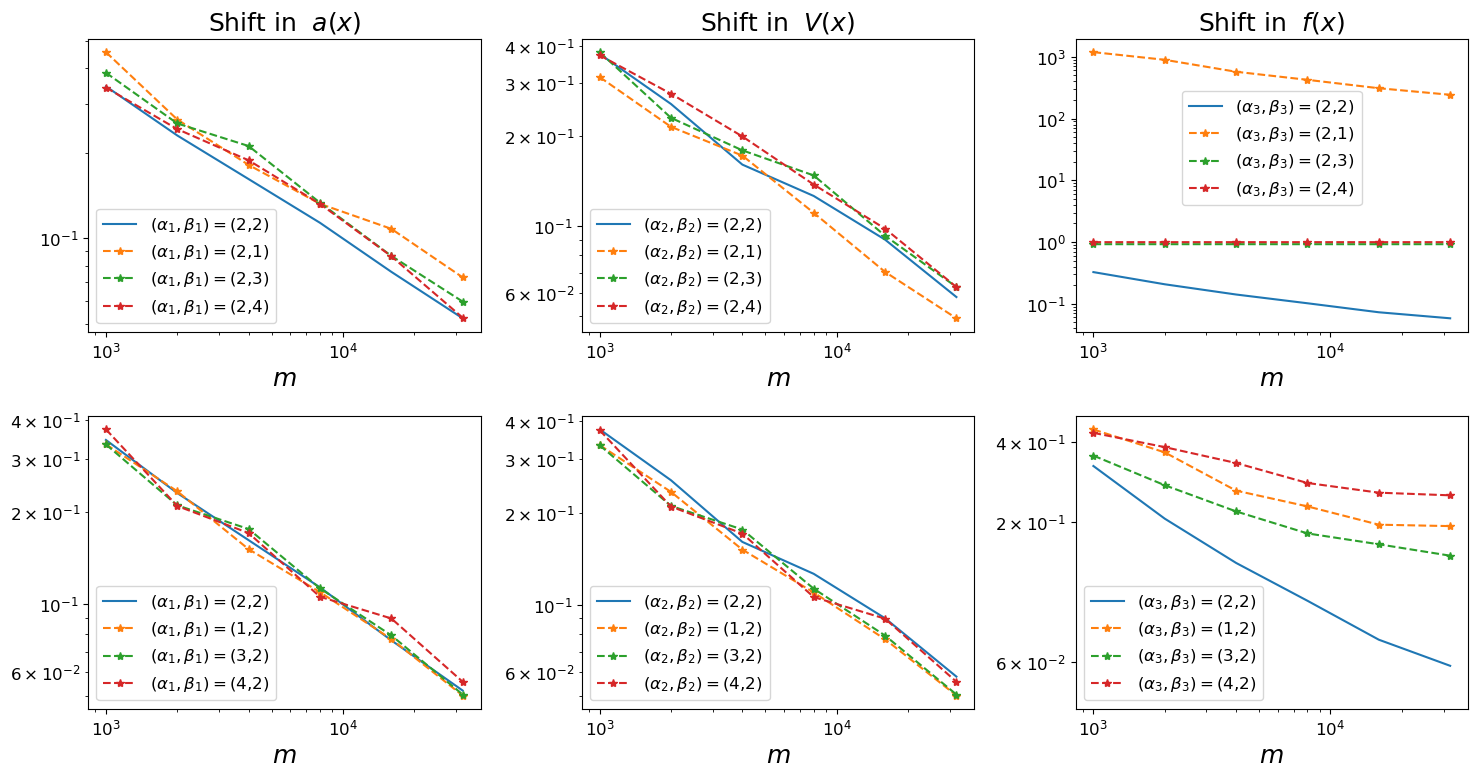}
  \caption{Out-of-domain generalization test with $d=32$, $n=2000$ and $N=20000$. Training is performed with $(\alpha_1, \beta_1)= (2, 2)$, $(\alpha_2, \beta_2)= (2, 2)$ and $(\alpha_3, \beta_3)= (2, 2)$. Each columns shows the relative $H^1$ error under the distribution shift on $a(x)$, $V(x)$ and $f(x)$, respectively, with respect to varying inference prompt length $m$.}
  \label{fig:PDE_ood_diverse}
\end{figure}

\section{Conclusion}
We studied the ability of single-layer linear transformers to solve linear systems in-context. We proved quantitative bounds on the in-domain generalization error of pre-trained transformers in terms of the various sample sizes. Concerning OOD generalization, our analysis revealed that the behavior of pre-trained transformers under task shifts is more subtle than the results for learning scalar-valued functions would suggest. To control the OOD generalization error under task shifts, we derived a novel notion of task diversity which defines a necessary and sufficient condition for pre-trained transformers to achieve OOD generalization. We also proved sufficient conditions on the distribution $P_{\A}$ under which task diversity holds. Finally, we leveraged our results to prove error bounds for in-context operator learning. 

Several important questions remain open for future investigation. First, it remains to verify the task diversity conditions in cases of interest. Second, we would like to study the in-context learning of nonlinear vector-valued functions, particularly those arising from discretizations of nonlinear operators. In this setting, it is important to characterize the role that depth and nonlinearity can play in the ability of transformers to approximate more complex functions, and to develop an appropriate notion of task diversity. Finally, while most of the existing results on in-context learning hold only for IID data, many scientific problems are inherently time-dependent. In these settings, the behavior of in-context learning is quite unexplored. We leave these directions to future work.

\bibliographystyle{abbrv}
\bibliography{refs}

\begin{thebibliography}{10}

\bibitem{achiam2023gpt}
J.~Achiam, S.~Adler, S.~Agarwal, L.~Ahmad, I.~Akkaya, F.~L. Aleman, D.~Almeida, J.~Altenschmidt, S.~Altman, S.~Anadkat, et~al.
\newblock {GPT}-4 technical report.
\newblock {\em arXiv preprint arXiv:2303.08774}, 2023.

\bibitem{ahn2023transformers}
K.~Ahn, X.~Cheng, H.~Daneshmand, and S.~Sra.
\newblock Transformers learn to implement preconditioned gradient descent for in-context learning.
\newblock {\em Advances in Neural Information Processing Systems}, 36:45614--45650, 2023.

\bibitem{ahn2023linear}
K.~Ahn, X.~Cheng, M.~Song, C.~Yun, A.~Jadbabaie, and S.~Sra.
\newblock Linear attention is (maybe) all you need (to understand transformer optimization).
\newblock {\em arXiv preprint arXiv:2310.01082}, 2023.

\bibitem{akyurek2022learning}
E.~Aky{\"u}rek, D.~Schuurmans, J.~Andreas, T.~Ma, and D.~Zhou.
\newblock What learning algorithm is in-context learning? investigations with linear models.
\newblock In {\em The Eleventh International Conference on Learning Representations}, 2022.

\bibitem{anwar2024adversarial}
U.~Anwar, J.~Von~Oswald, L.~Kirsch, D.~Krueger, and S.~Frei.
\newblock Adversarial robustness of in-context learning in transformers for linear regression.
\newblock {\em arXiv preprint arXiv:2411.05189}, 2024.

\bibitem{bai2024transformers}
Y.~Bai, F.~Chen, H.~Wang, C.~Xiong, and S.~Mei.
\newblock Transformers as statisticians: Provable in-context learning with in-context algorithm selection.
\newblock {\em Advances in neural information processing systems}, 36, 2024.

\bibitem{bartlett2005local}
P.~L. Bartlett, O.~Bousquet, and S.~Mendelson.
\newblock Local rademacher complexities.
\newblock 2005.

\bibitem{boffi2010finite}
D.~Boffi.
\newblock Finite element approximation of eigenvalue problems.
\newblock {\em Acta numerica}, 19:1--120, 2010.

\bibitem{brenner2007mathematical}
S.~Brenner and R.~Scott.
\newblock {\em The Mathematical Theory of Finite Element Methods}, volume~15.
\newblock Springer Science \& Business Media, 2007.

\bibitem{chen2024training}
S.~Chen, H.~Sheen, T.~Wang, and Z.~Yang.
\newblock Training dynamics of multi-head softmax attention for in-context learning: Emergence, convergence, and optimality.
\newblock {\em arXiv preprint arXiv:2402.19442}, 2024.

\bibitem{chen2024unveiling}
S.~Chen, H.~Sheen, T.~Wang, and Z.~Yang.
\newblock Unveiling induction heads: Provable training dynamics and feature learning in transformers.
\newblock {\em arXiv preprint arXiv:2409.10559}, 2024.

\bibitem{cole2025context}
F.~Cole, Y.~Lu, T.~Zhang, and Y.~Zhao.
\newblock In-context learning of linear dynamical systems with transformers: Error bounds and depth-separation.
\newblock {\em arXiv preprint arXiv:2502.08136}, 2025.

\bibitem{dudley1967sizes}
R.~M. Dudley.
\newblock The sizes of compact subsets of hilbert space and continuity of gaussian processes.
\newblock {\em Journal of Functional Analysis}, 1(3):290--330, 1967.

\bibitem{edelman2024evolution}
E.~Edelman, N.~Tsilivis, B.~Edelman, E.~Malach, and S.~Goel.
\newblock The evolution of statistical induction heads: In-context learning markov chains.
\newblock {\em Advances in Neural Information Processing Systems}, 37:64273--64311, 2024.

\bibitem{ern2004theory}
A.~Ern and J.-L. Guermond.
\newblock {\em Theory and practice of finite elements}, volume 159.
\newblock Springer, 2004.

\bibitem{garg2022can}
S.~Garg, D.~Tsipras, P.~S. Liang, and G.~Valiant.
\newblock What can transformers learn in-context? a case study of simple function classes.
\newblock {\em Advances in Neural Information Processing Systems}, 35:30583--30598, 2022.

\bibitem{goel2024can}
G.~Goel and P.~Bartlett.
\newblock Can a transformer represent a kalman filter?
\newblock In {\em 6th Annual Learning for Dynamics \& Control Conference}, pages 1502--1512. PMLR, 2024.

\bibitem{guo2023transformers}
T.~Guo, W.~Hu, S.~Mei, H.~Wang, C.~Xiong, S.~Savarese, and Y.~Bai.
\newblock How do transformers learn in-context beyond simple functions? a case study on learning with representations.
\newblock {\em arXiv preprint arXiv:2310.10616}, 2023.

\bibitem{huang2023context}
Y.~Huang, Y.~Cheng, and Y.~Liang.
\newblock In-context convergence of transformers.
\newblock {\em arXiv preprint arXiv:2310.05249}, 2023.

\bibitem{khan2022transformers}
S.~Khan, M.~Naseer, M.~Hayat, S.~W. Zamir, F.~S. Khan, and M.~Shah.
\newblock Transformers in vision: A survey.
\newblock {\em ACM computing surveys (CSUR)}, 54(10s):1--41, 2022.

\bibitem{kim2024task}
J.~Kim, S.~Kwon, J.~Y. Choi, J.~Park, J.~Cho, J.~D. Lee, and E.~K. Ryu.
\newblock Task diversity shortens the icl plateau.
\newblock {\em arXiv preprint arXiv:2410.05448}, 2024.

\bibitem{kim2024transformers}
J.~Kim, T.~Nakamaki, and T.~Suzuki.
\newblock Transformers are minimax optimal nonparametric in-context learners.
\newblock {\em arXiv preprint arXiv:2408.12186}, 2024.

\bibitem{kim2024transformers2}
J.~Kim and T.~Suzuki.
\newblock Transformers learn nonlinear features in context: Nonconvex mean-field dynamics on the attention landscape.
\newblock {\em arXiv preprint arXiv:2402.01258}, 2024.

\bibitem{li2023transformers}
Y.~Li, M.~E. Ildiz, D.~Papailiopoulos, and S.~Oymak.
\newblock Transformers as algorithms: Generalization and stability in in-context learning.
\newblock In {\em International conference on machine learning}, pages 19565--19594. PMLR, 2023.

\bibitem{li2024one}
Z.~Li, Y.~Cao, C.~Gao, Y.~He, H.~Liu, J.~M. Klusowski, J.~Fan, and M.~Wang.
\newblock One-layer transformer provably learns one-nearest neighbor in context.
\newblock {\em arXiv preprint arXiv:2411.10830}, 2024.

\bibitem{liu2021swin}
Z.~Liu, Y.~Lin, Y.~Cao, H.~Hu, Y.~Wei, Z.~Zhang, S.~Lin, and B.~Guo.
\newblock Swin transformer: Hierarchical vision transformer using shifted windows.
\newblock In {\em Proceedings of the IEEE/CVF international conference on computer vision}, pages 10012--10022, 2021.

\bibitem{lu2024asymptotic}
Y.~M. Lu, M.~I. Letey, J.~A. Zavatone-Veth, A.~Maiti, and C.~Pehlevan.
\newblock Asymptotic theory of in-context learning by linear attention.
\newblock {\em arXiv preprint arXiv:2405.11751}, 2024.

\bibitem{mahankali2023one}
A.~Mahankali, T.~B. Hashimoto, and T.~Ma.
\newblock One step of gradient descent is provably the optimal in-context learner with one layer of linear self-attention.
\newblock {\em arXiv preprint arXiv:2307.03576}, 2023.

\bibitem{mccabe2023multiple}
M.~McCabe, B.~R.-S. Blancard, L.~H. Parker, R.~Ohana, M.~Cranmer, A.~Bietti, M.~Eickenberg, S.~Golkar, G.~Krawezik, F.~Lanusse, et~al.
\newblock Multiple physics pretraining for physical surrogate models.
\newblock {\em arXiv preprint arXiv:2310.02994}, 2023.

\bibitem{mendelson2008obtaining}
S.~Mendelson.
\newblock Obtaining fast error rates in nonconvex situations.
\newblock {\em Journal of Complexity}, 24(3):380--397, 2008.

\bibitem{mroueh2023towards}
Y.~Mroueh.
\newblock Towards a statistical theory of learning to learn in-context with transformers.
\newblock In {\em NeurIPS 2023 Workshop Optimal Transport and Machine Learning}, 2023.

\bibitem{nichani2024transformers}
E.~Nichani, A.~Damian, and J.~D. Lee.
\newblock How transformers learn causal structure with gradient descent.
\newblock {\em arXiv preprint arXiv:2402.14735}, 2024.

\bibitem{oko2024pretrained}
K.~Oko, Y.~Song, T.~Suzuki, and D.~Wu.
\newblock Pretrained transformer efficiently learns low-dimensional target functions in-context.
\newblock {\em Advances in Neural Information Processing Systems}, 37:77316--77365, 2024.

\bibitem{olsson2022context}
C.~Olsson, N.~Elhage, N.~Nanda, N.~Joseph, N.~DasSarma, T.~Henighan, B.~Mann, A.~Askell, Y.~Bai, A.~Chen, et~al.
\newblock In-context learning and induction heads.
\newblock {\em arXiv preprint arXiv:2209.11895}, 2022.

\bibitem{raventos2023pretraining}
A.~Ravent{\'o}s, M.~Paul, F.~Chen, and S.~Ganguli.
\newblock Pretraining task diversity and the emergence of non-bayesian in-context learning for regression.
\newblock {\em Advances in neural information processing systems}, 36:14228--14246, 2023.

\bibitem{rudelson2013hanson}
M.~Rudelson and R.~Vershynin.
\newblock Hanson-wright inequality and sub-gaussian concentration.
\newblock 2013.

\bibitem{sander2024transformers}
M.~E. Sander, R.~Giryes, T.~Suzuki, M.~Blondel, and G.~Peyr{\'e}.
\newblock How do transformers perform in-context autoregressive learning?
\newblock {\em arXiv preprint arXiv:2402.05787}, 2024.

\bibitem{shalev2014understanding}
S.~Shalev-Shwartz and S.~Ben-David.
\newblock {\em Understanding machine learning: From theory to algorithms}.
\newblock Cambridge university press, 2014.

\bibitem{strang2022introduction}
G.~Strang.
\newblock {\em Introduction to linear algebra}.
\newblock SIAM, 2022.

\bibitem{subramanian2024towards}
S.~Subramanian, P.~Harrington, K.~Keutzer, W.~Bhimji, D.~Morozov, M.~W. Mahoney, and A.~Gholami.
\newblock Towards foundation models for scientific machine learning: Characterizing scaling and transfer behavior.
\newblock {\em Advances in Neural Information Processing Systems}, 36, 2024.

\bibitem{sun2024towards}
J.~Sun, Y.~Liu, Z.~Zhang, and H.~Schaeffer.
\newblock Towards a foundation model for partial differential equation: Multi-operator learning and extrapolation.
\newblock {\em arXiv preprint arXiv:2404.12355}, 2024.

\bibitem{van2000asymptotic}
A.~W. Van~der Vaart.
\newblock {\em Asymptotic statistics}, volume~3.
\newblock Cambridge university press, 2000.

\bibitem{vaswani2017attention}
A.~Vaswani, N.~Shazeer, N.~Parmar, J.~Uszkoreit, L.~Jones, A.~N. Gomez, {\L}.~Kaiser, and I.~Polosukhin.
\newblock Attention is all you need.
\newblock {\em Advances in neural information processing systems}, 30, 2017.

\bibitem{von2023transformers}
J.~Von~Oswald, E.~Niklasson, E.~Randazzo, J.~Sacramento, A.~Mordvintsev, A.~Zhmoginov, and M.~Vladymyrov.
\newblock Transformers learn in-context by gradient descent.
\newblock In {\em International Conference on Machine Learning}, pages 35151--35174. PMLR, 2023.

\bibitem{wainwright2019high}
M.~J. Wainwright.
\newblock {\em High-dimensional statistics: A non-asymptotic viewpoint}, volume~48.
\newblock Cambridge university press, 2019.

\bibitem{wu2023many}
J.~Wu, D.~Zou, Z.~Chen, V.~Braverman, Q.~Gu, and P.~L. Bartlett.
\newblock How many pretraining tasks are needed for in-context learning of linear regression?
\newblock {\em arXiv preprint arXiv:2310.08391}, 2023.

\bibitem{wu2024transformers}
W.~Wu, M.~Su, J.~Y.-C. Hu, Z.~Song, and H.~Liu.
\newblock Transformers are deep optimizers: Provable in-context learning for deep model training.
\newblock {\em arXiv preprint arXiv:2411.16549}, 2024.

\bibitem{yang2023context}
L.~Yang, S.~Liu, T.~Meng, and S.~J. Osher.
\newblock In-context operator learning with data prompts for differential equation problems.
\newblock {\em Proceedings of the National Academy of Sciences}, 120(39):e2310142120, 2023.

\bibitem{yang2024pde}
L.~Yang and S.~J. Osher.
\newblock Pde generalization of in-context operator networks: A study on 1d scalar nonlinear conservation laws.
\newblock {\em arXiv preprint arXiv:2401.07364}, 2024.

\bibitem{yang2024context}
T.~Yang, Y.~Huang, Y.~Liang, and Y.~Chi.
\newblock In-context learning with representations: Contextual generalization of trained transformers.
\newblock {\em arXiv preprint arXiv:2408.10147}, 2024.

\bibitem{ye2024pdeformer}
Z.~Ye, X.~Huang, L.~Chen, H.~Liu, Z.~Wang, and B.~Dong.
\newblock Pdeformer: Towards a foundation model for one-dimensional partial differential equations.
\newblock {\em arXiv preprint arXiv:2402.12652}, 2024.

\bibitem{zhang2023trained}
R.~Zhang, S.~Frei, and P.~L. Bartlett.
\newblock Trained transformers learn linear models in-context.
\newblock {\em arXiv preprint arXiv:2306.09927}, 2023.

\bibitem{zhang2024context}
R.~Zhang, J.~Wu, and P.~L. Bartlett.
\newblock In-context learning of a linear transformer block: benefits of the mlp component and one-step gd initialization.
\newblock {\em arXiv preprint arXiv:2402.14951}, 2024.

\bibitem{zheng2024mesa}
C.~Zheng, W.~Huang, R.~Wang, G.~Wu, J.~Zhu, and C.~Li.
\newblock On mesa-optimization in autoregressively trained transformers: Emergence and capability.
\newblock {\em Advances in Neural Information Processing Systems}, 37:49081--49129, 2024.

\end{thebibliography}

\appendix
\section{Proofs}

\subsection{Proofs for Section \ref{subsec:indomain}}.
We provide a rigorous proof for Theorem \ref{thm:indomaingen}, whose structure was outlined in Section \ref{sec: pfsketches}. First, we introduce some notation. Recall that given parameters $\theta = (P,Q)$ the transformer prediction with respect to the prompt
$$ Z = \begin{pmatrix}
    x_1 & \dots & x_n & x_{n+1} \\ y_1 & \dots & X_n & 0
\end{pmatrix}
$$
is given by
    $$ y^{\theta}_{n+1} = P \cdot \frac{1}{n} \sum_{k=1}^{n} y_{k} x_{k}^T \cdot  Q x_{n+1} = P A^{-1} X_{n} Q x_{n+1}.
    $$
    Define the individual loss $\ell_{\theta}$ as a function of the task $A$, the covariance $X_n$, and the query $x_{n+1}$ according to 
    $$ \ell_{\theta}(A,X_n,x_{n+1}) = \left\| P A^{-1} X_{n} Q x_{n+1} - A^{-1} x_{n+1}\right\|^2.
    $$
Then the population and empirical risk functionals can be expressed as
$$ \mathcal{R}_n(\theta) = \E_{x_1, \dots, x_{n+1} \sim \mathcal{N}(0,\Sigma), A \sim P_{\A}} \left[\ell_{\theta}(A,X_n,x_{n+1}) \right]
$$
and 
$$ \mathcal{R}_{n,N}(\theta) = \frac{1}{N} \sum_{i=1}^{N} \ell_{\theta}(A_i,X_{n,i},x_{n+1,i}).
$$

Given iid random vectors $x_1, \dots, x_{n+1} \sim \mathcal{N}(0,\Sigma)$ and a deterministic parameter $t > 0$, define the event
\begin{align*}
    \mathcal{S}_t(x_1, \dots, x_{n+1}) = \left\{\|x_{n+1}\| \leq \sqrt{\trace(\Sigma)} + t, \; \|X_n\|_{\op} \leq \|\Sigma\|_{\op}\left( 1 + t + \sqrt{\frac{d}{n}}\right) \right\}.
\end{align*}
Define $\mathcal{R}_n^{t}(\theta)$ as the $t$-truncated version of the empirical risk:
\begin{align*}
    \mathcal{R}_n^t(\theta) = \E_{x_1, \dots, x_{n+1} \sim \mathcal{N}(0,\Sigma), \; A \sim P_A} \left[ \ell_{\theta}(A,X_n,x_{n+1}) \cdot \mathbf{1}_{\mathcal{S}_t}(x_1, \dots, x_{n+1}) \right].
\end{align*}
Define the $t$-truncated empirical risk $\mathcal{R}_{n,N}^{t}$ similarly. We have the following error decomposition lemma for the empirical risk minimizer.

\begin{lemma}\label{lem: generrordecomp}
    Let $\widehat{\theta} \in \textrm{arg} \min_{\|\theta\| \leq M} \mathcal{R}_{n,N}(\theta).$ Fix $t > 0$ and set $\theta^{\ast} \in \textrm{arg} \min_{\|\theta\| \leq M} \mathcal{R}_n^t(\theta).$ Then we have
    \begin{align*}
        \mathcal{R}_{m}(\widehat{\theta}) &\leq 2 \sup_{\|\theta\| \leq M} \left| \left(\mathcal{R}_n-\mathcal{R}_m \right)(\theta)\right| + \sup_{\|\theta\| \leq M} \left(\mathcal{R}_n - \mathcal{R}_n^t \right)(\theta) + \left(\mathcal{R}_{n}^t(\widehat{\theta}) - \mathcal{R}_n^{t}(\theta^{\ast}) \right) \\
        &+ \inf_{\|\theta^{\ast}\| \leq M} \mathcal{R}_m(\theta^{\ast}), 
    \end{align*}
\end{lemma}
\begin{proof}
    For any $t > 0$ and $\theta^{\ast}$ with $\|\theta^{\ast}\| \leq M$, decompose the error as
    \begin{align*}
        \mathcal{R}_m(\widehat{\theta}) &= \left(\mathcal{R}_m - \mathcal{R}_n \right)(\widehat{\theta}) + \left(\mathcal{R}_n - \mathcal{R}_n^t \right)(\widehat{\theta}) + \left( \mathcal{R}_n^t(\widehat{\theta}) - \mathcal{R}_{n}^t(\theta^{\ast}) \right) \\ 
        &+ \left(\mathcal{R}_n^{t} - \mathcal{R}_n \right)(\theta^{\ast}) + \left(\mathcal{R}_n - \mathcal{R}_m \right) (\theta^{\ast}) + \mathcal{R}_m(\theta^{\ast})
    \end{align*}
    The first and fifth terms can be bounded by $\sup_{\|\theta\| \leq M} \left| \left(\mathcal{R}_n-\mathcal{R}_m \right)(\theta)\right|$, and the second term can be be bounded by $\sup_{\|\theta\| \leq M} \left(\mathcal{R}_n - \mathcal{R}_n^t \right)(\theta).$ The fourth term is non-positive since $\mathcal{R}_n^t(\cdot) \leq \mathcal{R}_n(\cdot).$ This proves the claim.
    
\end{proof}

Having established a decomposition of the in-domain generalization error in Lemma \ref{lem: generrordecomp}, we now seek to bound each term individually over the course of several lemmas. We begin by bounding the context-mismatch error $\sup_{\|\theta\| \leq M} \left| \left(\mathcal{R}_n-\mathcal{R}_m \right)(\theta)\right|.$

\begin{lemma}\label{lem:contextmismatchbound}[Context-mismatch error bound]
    The bound 
    \begin{align*}
        \sup_{\|\theta\| \leq M} \left| \left(\mathcal{R}_n-\mathcal{R}_m \right)(\theta)\right| \leq 2M^4 c_{\A}^2 \max \left(\trace(\Sigma), \|\Sigma\|_{\op}^2 \right) \trace(\Sigma) \left| \frac{1}{n} - \frac{1}{m} \right|
    \end{align*}
    holds.
\end{lemma}

\begin{proof}
    Denote $\theta = (P,Q)$. Recall that, as a direct consequence of Lemma \ref{lem: popriskexpression}, we have
    \begin{align*} \mathcal{R}_n(\theta) &= \E_A \big[ \trace(A^{-1} \Sigma A^{-1}) - \trace(PA^{-1}\Sigma Q \Sigma A^{-1}) - \trace(A^{-1} \Sigma Q^T \Sigma A^{-1} P^T) \\ &+ \frac{n+1}{n} \trace(P A^{-1} \Sigma Q \Sigma Q^T \Sigma A^{-1} P^T) + \frac{\trace_{\Sigma}(Q \Sigma Q^T)}{n} \trace(P A^{-1} \Sigma A^{-1} P^T) \big],
    \end{align*}
    An analogous expression holds for $\mathcal{R}_m(\theta).$ Therefore, for $\theta$ satisfying $\|\theta\| = \max(\|P\|_{\textrm{op}}, \|Q\|_{\textrm{op}}) \leq M,$ we have the bound
    \begin{align*}
        \Big| \mathcal{R}_m(\theta) - \mathcal{R}_n(\theta) \Big| &= \Big| \frac{1}{n} - \frac{1}{m} \Big| \Big|\E_{A} \Big[ \trace(P A^{-1} \Sigma Q \Sigma Q^T \Sigma A^{-1} P^T)+ \trace_{\Sigma}(Q \Sigma Q^T) \trace(P A^{-1} \Sigma A^{-1} P^T) \Big] \Big| \\
        &\leq \Big| \frac{1}{n} - \frac{1}{m} \Big| \cdot 2M^4 c_{\A}^2 \max(\trace(\Sigma),\|\Sigma\|_{\textrm{op}}^2) \trace(\Sigma).
    \end{align*}
\end{proof}

Next, we bound the truncation error $\sup_{\|\theta\| \leq M} \left(\mathcal{R}_n - \mathcal{R}_n^t \right)(\theta).$

\begin{lemma}\label{lem: truncerrorbound}[Truncation error bound]
    The truncation error $\sup_{\|\theta\| \leq M} \left(\mathcal{R}_n - \mathcal{R}_n^t \right)(\theta)$ is bounded from above by
    \begin{align*}
        c_{\A}^2 \left(M^2 \E[\|X_n\|_{\op}^8]^{1/4} + \E[\|x_{n+1}\|^8]^{1/4} \right) \cdot \sqrt{2\exp \left(-\frac{nt^2}{2} \right) + 2\exp \left(-\frac{t^2}{C \|\Sigma\|_{\op}} \right)}.
    \end{align*}
\end{lemma}
\begin{proof}
    By Lemma \ref{gaussianconc} and Example 6.2 in \cite{wainwright2019high}, we have the concentration bound
    \begin{align*}
        \mathbb{P} \left(\mathcal{S}^c_t(x_1, \dots, x_{n+1}) \right) \leq 2\exp \left(-\frac{nt^2}{2} \right) + 2\exp \left(-\frac{t^2}{C \|\Sigma\|_{\op}} \right)
    \end{align*}
    for some universal constant $C > 0$. Therefore, for any $\|\theta\| \leq M$, we can apply the Cauchy-Schwarz inequality to obtain
    \begin{align*}
        &\mathcal{R}_n(\theta) - \mathcal{R}_n^t(\theta) = \E \left[ \left\| \left( P A^{-1} X_n Q - A^{-1}\right)x_{n+1} \right\|^2 \cdot \mathbf{1}_{\mathcal{S}_t}(x_1, \dots, x_{n+1}) \right] \\
        &\leq \left(\E \left\| \left( P A^{-1} X_n Q - A^{-1}\right)x_{n+1} \right\|^4 \right)^{1/2} \cdot \mathbb{P} \left( \mathcal{S}_{t}^c(x_1, \dots, x_{n+1})\right)^{1/2} \\
        &\leq c_{\A}^2 \left(M^2 \E[\|X_n\|_{\op}^8]^{1/4} + \E[\|x_{n+1}\|^8]^{1/4} \right) \cdot \sqrt{2\exp \left(-\frac{nt^2}{2} \right) + 2\exp \left(-\frac{t^2}{C \|\Sigma\|_{\op}} \right)},
    \end{align*}
    where we used the Cauchy-Schwarz inequality twice and the triangle inequality.
\end{proof}
This shows that the growth of the generalization error is quite mild with respect to $t$. In particular, the error decays like a Gaussian with respect to $t$, but the constant factors are only polynomial in problem parameters. Next, we bound the statistical error $\left(\mathcal{R}_{n}^t(\widehat{\theta}) - \mathcal{R}_n^{t}(\theta^{\ast}) \right).$ To this end, we first show that, on the event where $\widehat{\theta}$ minimizes the truncated empirical risk $\mathcal{R}_{n,N}^t(\cdot)$, the statistical error can be controlled in terms of the Rademacher complexity of an associated function class. We also show that $\widehat{\theta}$ minimizes the truncated empirical risk with high probability. Then we bound the Rademacher complexity of the relevant function class, leading to a bound on the statistical error. The bound depends on whether or not the structural condition in Definition \ref{def:fast} is satisfied. If the structural condition is not satisfied, the statistical error is bounded by the Rademacher complexity directly, leading to an $O\left( \frac{1}{\sqrt{N}} \right)$ bound, whereas if the structural condition is satisfied, the statistical error is bounded by the fixed point of the localized Rademacher complexity, which accelerates the rate to $O \left( \frac{1}{\sqrt{N}} \right).$

Before stating the relevant lemmas, we recall that if $\mathcal{F}$ is a class of real-valued functions on a set $\mathcal{X}$, the \textit{empirical Rademacher complexity} of $\mathcal{F}$ with respect to a collection of samples $\{x_i\}_{i=1}^{N} \subset \mathcal{X}$ is defined by
$$ \textrm{Rad}(\mathcal{F};\{x_i\}) = \E_{\epsilon_i} \sup_{f \in \mathcal{F}} \frac{1}{N} \sum_{i=1}^{N} \epsilon_i f(x_i),
$$
where $\{ \epsilon_i\}$ are iid Bernoulli random variables. When the samples $\{x_i\}$ follow a probability distribution, the \textit{Rademacher complexity} of $\mathcal{F}$ is the average of the empirical Rademacher complexity over the data distribution.

\begin{lemma}\label{lem: localradcomplexity}
    Suppose that $\widehat{\theta} \in \textrm{arg} \min_{\|\theta\| \leq M} \mathcal{R}_{n,N}^t(\theta)$, and recall that $\theta^{\ast} \in \textrm{arg} \min_{\| \theta\| \leq M} \mathcal{R}_n^t(\theta).$ Then the statistical error satisfies the bound
    \begin{align*}
    \left(\mathcal{R}_{n}^t(\widehat{\theta}) - \mathcal{R}_n^{t}(\theta^{\ast}) \right) &\lesssim \textrm{Rad}(\{\ell_{\theta}^t: \|\theta\| \leq M\}) + \sqrt{\frac{2 \log(1/\delta)}{N}} 
    \end{align*}
    with probability $1-\delta.$ If, in addition, the structural condition of Definition \ref{def:fast} holds, then we have the bound
    \begin{align*}
    \left(\mathcal{R}_{n}^t(\widehat{\theta}) - \mathcal{R}_n^{t}(\theta^{\ast}) \right) \lesssim r^{\ast} + \frac{\delta}{N},
    \end{align*}
    with probability $1-\delta.$ Here, $r^{\ast}$ is the fixed point of the function
    \begin{align*}
        r \mapsto \textrm{Rad}\left\{\ell_{\theta}^t-\ell_{\theta^{\ast}}^t: \|\ell_{\theta}^t-\ell_{\theta^{\ast}}^t\|_{L^2} \leq r \right\}.
    \end{align*}
    Moreover, the condition that $\widehat{\theta} \in \textrm{arg} \min_{\|\theta\| \leq M} \mathcal{R}_{n,N}^t(\theta)$ holds with probability $1-\delta_{N,n,t}$, where
    $$\delta_{N,n,t} = 2N \left( \exp \left(-\frac{nt^2}{2} \right) + \exp \left(  - \frac{t^2}{C \|\Sigma\|_{\op}}\right) \right)$$ for some absolute constant $C > 0.$
\end{lemma}
\begin{proof}
    The claim is a standard application of statistical learning theory. For the first bound, we have
    \begin{align*}
    \left(\mathcal{R}_{n}^t(\widehat{\theta}) - \mathcal{R}_n^{t}(\theta^{\ast}) \right) &= \left(\mathcal{R}_n^t(\widehat{\theta} - \mathcal{R}_{n,N}^t(\widehat{\theta} \right) + \left(\mathcal{R}_{n,N}^t(\widehat{\theta}) - \mathcal{R}_{n,N}^t(\theta^{\ast}) \right) \\
    &+ \left( \mathcal{R}_{n,N}^t(\theta^{\ast}) - \mathcal{R}_{n}^t(\theta^{\ast}) \right).
    \end{align*}
    On the event that $\widehat{\theta} \in \textrm{arg} \min_{\|\theta\| \leq M} \mathcal{R}_{n,N}^t(\theta)$, the second term is nonpositive by the optimality of $\widehat{\theta}.$ The first and third terms can be trivially bounded by $\sup_{\|\theta\| \leq M} \left|(\mathcal{R}_n^t - \mathcal{R}_{n,N}^t)(\theta) \right|$, which can be expressed in terms of the individual loss $\ell_{\theta}^t$ as
    \begin{align*}
        \sup_{\|\theta\| \leq M} \left| \E \ell_{\theta}^t(A,X_n,x_{n+1}) - \frac{1}{N} \sum_{i=1}^{N} \ell_{\theta}^t(A_i,X_{n,i},x_{n+1,i}) \right|.
    \end{align*}
    By Theorem 26.5 in \cite{shalev2014understanding}, we have with probability at least $1-\delta$,
    \begin{align}\label{radcomplexitybd}
        \sup_{\|\theta\| \leq M} \Big|\E_{A,X_n,x_{n+1}}\ell^t_{\theta}(A,X_n,x_{n+1}) - \sum_{i=1}^{N} \ell^t_{\theta}(A_i,X_{n,i},x_{n+1,i}) \Big| &\leq \textrm{Rad}_N(\{\ell^t_{\theta}: \|\theta\| \leq M\}) \\ &+ \sup_{\|\theta\| \leq M} \|\ell^t_{\theta}\|_{\infty} \cdot \sqrt{\frac{2\log(1/\delta)}{N}},
    \end{align}
    where the expectations over $X_n$ and $x_{n+1}$ are taken over the truncated versions of the original distributions. By the boundedness of the data, we have 
    $$ \|(PA_i^{-1} X_{n,i} Q - A_i^{-1})x_{n+1,i}\|^2 \leq M^4 c_A^2 \|\Sigma\|_{\textrm{op}}^2 \left(1 + t + \sqrt{\frac{d}{n}} \right)^2 \left(\sqrt{\trace(\Sigma)} + t \right)^2,
    $$
    and hence the second term in Inequality \eqref{radcomplexitybd} is bounded from above by
    $$ M^4 c_A^2 \|\Sigma\|_{\textrm{op}}^2 \left(1 + t + \sqrt{\frac{d}{n}} \right)^2 \left(\sqrt{\trace(\Sigma)} + t \right)^2 \cdot \sqrt{\frac{2\log(1/\delta)}{N}} = O \left( \sqrt{\frac{2\log(1/\delta)}{N}} \right).
    $$
    This proves the first claim. To prove the improved bound under the structural condition, note that the structural condition implies that the function class $\{\ell_{\theta} - \ell_{\theta^{\ast}: \|\theta\| \leq M}$ satisfies a Bernstein condition, in the sense that there is a constant $C > 0$ such that for each $\|\theta\| \leq M$, the variance (with respect to the truncated data distribution of $(A,X_n,x_{n_1})$) of $\ell_{\theta} - \ell_{\theta^{\ast}}$ is bounded by $C$ times the expectation of $\ell_{\theta} - \ell_{\theta^{\ast}}.$ This is proven in Theorem A of \cite{mendelson2008obtaining}. The desired bound on the statistical error then follows from Theorem 3.3 in \cite{bartlett2005local} and the optimality of $\widehat{\theta}$ and $\theta^{\ast}.$ Finally, the condition $\widehat{\theta} \in \textrm{arg} \min_{\|\theta\| \leq M} \mathcal{R}_{n,N}^t$ holds when the event $\mathcal{S}_t$ holds for each sample $(x_{1,i}, \dots, x_{n+1,i})$, because this ensures that $\mathcal{R}_{n,N}(\cdot) = \mathcal{R}_{n,N}^t(\cdot).$ By standard concentration inequalities for Gaussian random vectors and their empirical covariance matrices (see, e.g., Lemma \ref{gaussianconc} and Example 6.3 in \cite{wainwright2019high}), this holds with probability $1- \delta_{N,n,t}.$
\end{proof}

We now complete the statistical error bound by directly estimating the Rademacher complexity of the function class defined by the individual loss.

\begin{lemma}\label{lem: staterrorbd}
    We have the Rademacher complexity bound
    \begin{align*}
        \textrm{Rad}\{\ell_{\theta}^t: \|\theta\| \leq M\} = O \left(\frac{1}{\sqrt{N}} \right)
    \end{align*}
    and the local Rademacher complexity bound
    \begin{align*}
        \textrm{Rad}\{\ell_{\theta} - \ell_{\theta^{\ast}}: \|\theta\| \leq M, \; \|\ell_{\theta} - \ell_{\theta^{\ast}}\|_{L^2}^2 \leq r\} = O \left( \sqrt{\frac{r}{N}} \right).    
    \end{align*}
    Consequently, for any $\delta \in (0,1),$ the statistical error satifies the bounds
    \begin{align*}
        \left( \mathcal{R}_n^t(\widehat{\theta}) - \mathcal{R}_n^t(\theta^{\ast}) \right) = O \left(\frac{1+\sqrt{\delta}}{N} \right)
    \end{align*}
    and, under the structural condition,
    \begin{align*}
         \left( \mathcal{R}_n^t(\widehat{\theta}) - \mathcal{R}_n^t(\theta^{\ast}) \right) = O \left(\frac{1+\delta}{N} \right),
    \end{align*}
    with probability at least $1 - \delta_{n,N,t} - \delta$, where $\delta_{n,N,t}$ is as defined in Lemma \ref{lem: localradcomplexity}.
\end{lemma}
\begin{proof}
     Since we are working on the event $\bigcap_{i \in [N]} \mathcal{S}_t(x_{1,i}, \dots, x_{n+1,i})$, we can drop the dependence of the risk functionals on $t$. Notice that by the triangle inequality, we have
    \begin{align*}
        &\E_{\epsilon_i} \sup_{\|\theta\| \leq M} \frac{1}{N} \sum_{i=1}^{N} \epsilon_i \|(PA_i^{-1}X_{n,i}Q-A_i^{-1})x_{n+1,i}\|^2 \\
        &= \E_{\epsilon_i} \sup_{\|\theta\| \leq M} \frac{1}{N} \sum_{i=1}^{N} \epsilon_i \Big( \|PA_i^{-1}X_{n,i} Q x_{n+1,i}\|^2 + \|A_i^{-1}x_{n+1,i}\|^2 - 2 \langle P A_{i}^{-1} X_{n,i} Q x_{n+1,i}, A_i^{-1}x_{n+1,i} \rangle \Big) \\
        &\leq \E_{\epsilon_i} \sup_{\|\theta\| \leq M} \frac{1}{N} \sum_{i=1}^{N} \epsilon_i \|PA_i^{-1}X_{n,i} Q x_{n+1,i}\|^2 + 2 \E_{\epsilon_i} \sup_{\|\theta\| \leq M} \frac{1}{N} \sum_{i=1}^{N} \epsilon_i \langle P A_{i}^{-1} X_{n,i} Q x_{n+1,i}, A_i^{-1}x_{n+1,i} \rangle,
    \end{align*}
    where the last inequality follows from the triangle inequality, noting that the term $\sum_{i=1}^{N} \epsilon_i \|A_i^{-1} x_{n+1,i}\|^2$ is independent of $\theta$ and hence vanishes in the expectation over $\epsilon_i.$ Now, define the function classes
    \begin{align*} \Theta_1(M) &= \{(A,X_n,x_{n+1}) \mapsto \|PA^{-1} X_n Q x_{n+1}\|^2: \|\theta\| \leq M \}, \\ \Theta_2(M) &= \{(A,X_n,x_{n+1}) \mapsto \langle PA^{-1} X_n Q x_{n+1}, A^{-1}x_{n+1} \rangle: \|\theta\| \leq M \}.
    \end{align*}
    By Dudley's integral theorem \cite{dudley1967sizes}, it holds that
    \begin{equation}\label{dudley} \E_{\epsilon_i} \sup_{\|\theta\| \leq M} \frac{1}{N} \sum_{i=1}^{N} \epsilon_i \|PA_i^{-1}X_{n,i} Q x_{n+1,i}\|^2 \leq \inf_{\epsilon > 0} \frac{12\sqrt{2}}{\sqrt{N}} \int_{\epsilon}^{D_1(M)} \sqrt{\log \mathcal{N}\Big(\Theta_1(M), \| \cdot \|_N, \tau \Big)} d\tau,
    \end{equation}
    where $\mathcal{N}\Big(\Theta_1(M), \| \cdot \|_N, \tau \Big)$ is the $\tau$-covering number of the function class $\Theta_1(M)$ with respect to the metric induced by the empirical $L^2$ norm $\|F\|_N^2 = \frac{1}{N} \sum_{i=1}^{N} F(A_i,X_{n,i},x_{n+1,i})^2$ and
    $$ D_1(M) = \sup_{\|\theta\| \leq M} \Big\| \|PA^{-1} X_n Q x_{n+1}\|^2 \Big\|_N. 
    $$
        Note the bound
    \begin{align*}
        D_1(M)^2 &= \sup_{\|\theta\| \leq M} \frac{1}{N} \sum_{i=1}^{N} \|PA_i^{-1} X_{n,i} Q x_{n+1,i}\|^4  \\
        &\leq \frac{1}{N} \sum_{i=1}^{N} M^8 c_A^4 \|\Sigma\|_{\textrm{op}}^4\Big(1+t+\sqrt{\frac{d}{n}} \Big)^4 \Big(\sqrt{\trace(\Sigma)}+t \Big)^4
    \end{align*}
    and hence $D_1(M) \leq M^4 c_A^2 \|\Sigma\|_{\textrm{op}}^2 \Big(1+t+\sqrt{\frac{d}{n}} \Big)^2 \Big(\sqrt{\trace(\Sigma)}+t \Big)^2$. Similarly, for $\theta_1 = (P_1,Q_1), \theta_2 = (P_2,Q_2)$, with $\|\theta_1\|, \|\theta_2\| \leq M$, we have
    \begin{align*}
        \|\theta_1 - \theta_2\|_{N}^2 &= \frac{1}{N} \sum_{i=1}^{N} \|(P_1-P_2) A_i^{-1} X_{n,i} (Q_1-Q_2)x_{n+1,i}\|^4 \\
        &\leq 16 M^4 c_A^2 \|\Sigma\|_{\textrm{op}}^2 \Big(1+t+\sqrt{\frac{d}{n}} \Big)^2 R^2 \cdot \frac{1}{N} \sum_{i=1}^{N} \|(P_1-P_2) A_i^{-1} X_{n,i} (Q_1-Q_2)\|^2 \\
        &\leq M^4 c_A^4 \|\Sigma\|_{\textrm{op}}^4 \Big(1+t+\sqrt{\frac{d}{n}} \Big)^4 \Big(\sqrt{\trace(\Sigma)}+t \Big)^4 \cdot \max \Big(\|P_1-P_2\|_{\textrm{op}}^2, \|Q_1-Q_2\|_{\textrm{op}}^2 \Big).
    \end{align*}
    This shows that the metric induced by $\| \cdot \|_N$ is dominated by the metric $d(\theta_1,\theta_2) = \max\Big( \|P_1-P_2\|_{\textrm{op}}, \|Q_1-Q_2\|_{\textrm{op}} \Big)$, up to a factor of $M^2 c_A^2 \|\Sigma\|_{\textrm{op}}^2 \Big(1+t+\sqrt{\frac{d}{n}} \Big)^2 \Big(\sqrt{\trace(\Sigma)}+t \Big)^2$. The covering number of the set $\{\|\theta\| \leq M\}$ in the metric $d(\cdot,\cdot)$ is well-known, from which we conclude that
    $$ \log \mathcal{N}\Big(\Theta_1(M), \| \cdot \|_N, \tau \Big) \leq 2d^2 \log \Big(M^2 c_A^2 \|\Sigma\|_{\textrm{op}}^2 \Big(1 + \frac{2}{\tau} \Big) \Big).
    $$
    Optimizing over the choice of $\epsilon$ in Equation \eqref{dudley}, this proves that
    \begin{align}\label{staterrorbd} & \E_{\epsilon_i} \sup_{\|\theta\| \leq M} \frac{1}{N} \sum_{i=1}^{N} \epsilon_i \|PA_i^{-1}X_{n,i} Q x_{n+1,i}\|^2\\ &= O \Big( \frac{d^2 M^4 c_A^2 \|\Sigma\|_{\textrm{op}}^2 \Big(1+t+\sqrt{\frac{d}{n}} \Big)^2 \Big(\sqrt{\trace(\Sigma)}+t \Big)^2}{\sqrt{N}} \Big),
    \end{align}
    where $O( \cdot)$ omits factors that are logarithmic in $N$. An analogous argument proves a bound of the same order on the quantity
    $$ \E_{\epsilon_i} \sup_{\|\theta\| \leq M} \frac{1}{N} \sum_{i=1}^{N} \epsilon_i \langle P A_{i}^{-1} X_{n,i} Q x_{n+1,i}, A_i^{-1}x_{n+1,i} \rangle,
    $$
    which in turn bounds the Rademacher complexity $\textrm{Rad}_N(\{\ell_{\theta}: \|\theta\| \leq M\})$ by the right-hand side of Equation \eqref{staterrorbd}. Combining with Lemma \ref{lem: localradcomplexity}, this proves that the statistical error is bounded by $O \left(\frac{1}{\sqrt{N}} \right)$. When the structural condition is satisfied, the statistical error is bounded, up to $O(\left( \frac{1}{N} \right)$ terms, by the fixed point $r^{\ast}$ of the local Rademacher complexity $\left\{\ell_{\theta}-\ell_{\theta^{\ast}}: \|\theta\| \leq M, \; \|\ell_{\theta}-\ell_{\theta^{\ast}}\|^2_{L^2} \leq r \right\}.$ The Rademacher complexity of this localized space can be bounded using techniques analogous to those we used to bound the Rademacher complexity of the overall function class. The key difference is that, since the $L^2$-diameter of the localized function class is bounded by $\sqrt{r}$, the local Rademacher complexity inherits a factor of $\sqrt{r}$, owed to the fact that the upper limit of the integral from Dudley's chaining bound is bounded by $\sqrt{r}$. In turn, this implies that the localized Rademacher complexity is bounded by $O \left(\sqrt{\frac{r}{N}} \right).$ It is clear that the fixed point $r^{\ast}$ of the local Rademacher complexity is also $O\left( \frac{1}{\sqrt{N}} \right)$. Combined with Lemma \ref{lem: localradcomplexity}, this implies that the statistical error is bounded by $O \left( \frac{1}{N} \right)$ when the structural condition is satisfied.
\end{proof}

Finally, we must control the approximation error $\mathcal{R}_m(\theta^{\ast})$ where $\theta^{\ast} \in \textrm{arg} \min_{\| \theta\| \leq M} \mathcal{R}_n^t(\theta).$ When $t$ is sufficiently large, $\mathcal{R}_n^t \approx \mathcal{R}_n$ and the error between their minimizers is negligible, so it suffices to consider $\theta^{\ast} \in \textrm{arg} \min_{\|\theta\| \leq M} \mathcal{R}_n(\theta).$ While it is difficult to the minimizer $\theta^{\ast}$ directly, we can upper bound its approximation error by constructing a particular transformer. 
\begin{lemma}\label{lem: approxerrorbd}
    Let $B := \E_{A \sim P_{\A}}[A^{-2}].$ For $n \in \mathbb{N}$, define
    $$ Q_n = B\Big(\frac{n+1}{n} \Sigma B + \frac{\trace_{\Sigma}(B)}{n} \Sigma \Big)^{-1}.
    $$
    Then
    \begin{enumerate}
        \item $Q_n$ is the minimizer of the $Q \mapsto \mathcal{R}_n(\mathbf{I}_d,Q)$.
        \item For any $m \in \mathbb{N}$, $Q_n$ satisfies
        \begin{align*}
            \mathcal{R}_m(\mathbf{I}_d,Q_n) \leq \frac{c_A^2 \trace(\Sigma)}{m} + \frac{c_A^6 \|\Sigma^{-1}\|_{\textrm{op}}^{2} \|\Sigma\|_{\textrm{op}}^6 \trace(\Sigma) }{n^2} + O \Big( \frac{1}{mn} \Big).
        \end{align*}
    \end{enumerate}
\end{lemma}
The characterization of $Q_n$ as the minimizer of the $n$-shot population does not play a role in the estimate of the theorem above, but it does explain the choice of $Q_n$ to bound the approximation error. For a general task distribution $P_{\A}$, it may not be possible to analytically describe the minimizers of $\mathcal{R}_n$ when both $P$ and $Q$ are allowed to vary. However, our theory of task diversity, as well as our numerical results, suggest that taking $P = \mathbf{I}_d$ is a reasonable ansatz. When $P$ is fixed, the problem of optimizing $Q$ becomes convex.
\begin{proof}
    To prove 1), first recall the definition of the population risk functional
    $$ \mathcal{R}_n(\mathbf{I}_d,Q) = \E\Big[ \Big\| A^{-1} \Big(X_n Q - I \Big)y \Big\|^2 \Big],
    $$
    where $X_n := \frac{1}{n} \sum_{i=1}^{n} x_i x_i^T$ denotes the empirical covariance of $\{x_i\}_{i=1}^{n}$. Note that, conditioned on $A$ and $\{x_i\}_{i=1}^{n}$, $ A^{-1} \Big(X_n Q - I \Big)y$ is a centered Gaussian random vector with covariance $A^{-1} \Big(X_n Q - I \Big) \Sigma \Big(Q X_n - I \Big) A^{-1}.$ In addition, since the task and data distributions are independent, we can replace the task by its expectation. It therefore holds that
    \begin{align*}
        \E\Big[ \Big\| A^{-1} \Big(X_n Q - I \Big)y \Big\|^2 \Big] &= \E_{X_n} \Big[\trace \Big( B \Big(X_n Q - I \Big) \Sigma \Big(Q^T X_n - I \Big) \Big) \Big].
    \end{align*}

    Since this is a convex functional of $Q$, it suffices to characterize the critical point. Taking the derivative, we find that the critical point equation for the risk it
    $$ \nabla_Q \mathcal{R}(\mathbf{I}_d,Q) = \E_{X_n}[\Sigma Q^T X_n B X_n + X_n B X_n Q \Sigma] -2\Sigma B \Sigma =0.
    $$
    Using Lemma \ref{technicallemma} to compute the expectation, we further rewrite the critical point equation as
    $$ \Big(\frac{n+1}{n} B \Sigma + \frac{\trace_{\Sigma}(B)}{n} \Sigma \Big) Q + Q^T \Big( \frac{n+1}{n} \Sigma B + \frac{\trace(\Sigma)}{n} \Sigma \Big) = 2B.
    $$
    This equation is solved by the matrix $Q_n$ defined in the statement of the lemma. To prove 2), note that by Lemma \ref{perturbationlemma}, we can write $Q_n = \Sigma^{-1} + \frac{1}{n}K,$ where \begin{equation}\label{boundonK}
    \|K\|_{\textrm{op}} \leq \|\Sigma^{-1}\|_{\textrm{op}} \|\Sigma\|_{\textrm{op}} \Big(1 + \trace_{\Sigma}(B) \Big) C_A^2.
    \end{equation}
    It follows that
    \begin{align*}
        \mathcal{R}_m(\mathbf{I}_d,Q_n) &= \E_{A,X_m}[\trace(A^{-1}(X_m Q_n - \mathbf{I}_d)\Sigma(Q_n X_m - \mathbf{I}_d)A^{-1})] \\
        &= \E_{X_m}[\trace(B(X_m Q_n - \mathbf{I}_d)\Sigma(Q_n^T X_m - \mathbf{I}_d))], \; \; B := \E[A^{-2}] \\
        &= \trace(B \Sigma) + \E_{X_m}[\trace(BX_m Q_n \Sigma Q_n^T X_m)] -  \trace(B \Sigma Q_n \Sigma) - \trace(B \Sigma Q_n^T \Sigma) \\
        &= \trace(B \Sigma) + \trace(B \Sigma Q_n \Sigma Q_n \Sigma) - \trace(B \Sigma Q_n \Sigma) - \trace(B \Sigma Q_n^T \Sigma) \\ &+ \frac{1}{m} \Big(\trace \Big( B \Sigma Q_n \Sigma Q_n^T \Sigma \Big) + \trace_{\Sigma}(Q_n \Sigma Q_n^T) \trace(B \Sigma) \Big)
    \end{align*}
    where the last equality follows from Lemma \ref{technicallemma}. Writing $Q_n = \Sigma^{-1} + \frac{1}{n} K$ and doing some simplifying algebra, we find that
    \begin{align*}
        \mathcal{R}_m(\mathbf{I}_d,Q_n) &= \frac{1}{m} \Big( \trace((B+\trace_{\Sigma}(\Sigma^{-1} \mathbf{I_d}) \Sigma) \Big) + \frac{1}{n^2} \trace \Big(B \Sigma K \Sigma K^T \Sigma \Big) + O \Big( \frac{1}{mn} \Big) \\
        &= \frac{1}{m} \Big( \trace((B+d \mathbf{I_d}) \Sigma) \Big) + \frac{1}{n^2} \trace \Big(B \Sigma K \Sigma K^T \Sigma \Big) + O \Big( \frac{1}{mn} \Big),
    \end{align*}
    where we used the fact that $\trace_{\Sigma}(\Sigma^{-1}) = d$. Using the bound on the norm of $K$ stated in Equation \eqref{boundonK}, and the fact that $\|B\|_{\textrm{op}} \leq c_A^2,$ we have
    $$ \trace(B \Sigma K \Sigma K^T \Sigma) \leq c_A^2 C_A^4 \|\Sigma\|_{\textrm{op}}^2 \|\Sigma^{-1}\|_{\textrm{op}}^2 \Big(1 + \trace_{\Sigma}(B) \Big)^2 \trace(\Sigma).
    $$
    Similarly, the bound
    $$ \trace((B+d \mathbf{I_d}) \Sigma) \leq (c_A^2 + d)\trace(\Sigma)
    $$
    holds.
    We conclude that
    $$ \mathcal{R}_m(\mathbf{I}_d, Q_n) \leq \frac{(c_A^2 +d)\trace(\Sigma)}{m} + \frac{c_A^2 C_A^4 \|\Sigma\|_{\textrm{op}}^2 \|\Sigma^{-1}\|_{\textrm{op}}^2 \Big(1 + \trace_{\Sigma}(B) \Big)^2 \trace(\Sigma)}{n^2} + O \Big( \frac{1}{mn} \Big).
    $$
\end{proof}

At this point, the proof of Theorem \ref{thm:indomaingen} follows simply by combining the previous lemmas.

\begin{proof}[Proof of Theorem \ref{thm:indomaingen}]
    We use Lemma \ref{lem: generrordecomp} to decompose the error into a sum of truncation error, context mismatch error, statistical error, and approximation error. Lemma \ref{lem:contextmismatchbound} bounds the context mismatch error by $O \left( \left| \frac{1}{n} - \frac{1}{m} \right| \right)$; Lemma \ref{lem: truncerrorbound} bounds the truncation error by $O \left( \sqrt{2\exp \left(-\frac{nt^2}{2} \right) + 2\exp \left(-\frac{t^2}{C \|\Sigma\|_{\op}} \right)} \right)$; Lemma \ref{lem: staterrorbd} controls the statistical error by $O \left(\frac{1+\sqrt{\delta}}{\sqrt{N}} \right)$ in the general case and $O \left( \frac{1+\delta}{N}\right)$ under the structural condition, with probability $1-\delta-\delta_{N,n,t}$, and Lemma \ref{lem: approxerrorbd} bounds the approximation error by $O \left(\frac{1}{m} + \frac{1}{n^2} \right).$ By setting $t, \delta = O\left(\textrm{polylog}(N) \right)$, we can take the truncation error to be $O \left(\frac{1}{N^2} \right)$. Additionally, under the assumption that $m \leq n$, the context mismatch error is bounded by $O \left( \frac{1}{m} \right)$, and hence can be absorbed by the approximation error. Thus, the statistical and approximation errors dominate, and, up to logarithmic factors in $N$ and constants which are polynomial in problem parameters, the final bound is
    \begin{align*}
        \mathcal{R}_m(\widehat{\theta}) = \begin{cases}
            O \left(\frac{1}{m} + \frac{1}{n^2} + \frac{1}{\sqrt{N}} \right), \; \textrm{in general} \\
            O \left(\frac{1}{m} + \frac{1}{n^2} + \frac{1}{N} \right), \; \textrm{under the structural condition}
        \end{cases}
    \end{align*}
    with probability $1 - \frac{1}{\textrm{poly}(N)}.$
\end{proof}

\section{Proofs for Section \ref{subsec:ood}}
Recall that in Section \ref{sec: pfsketches}, we already proved Theorems \ref{thm: taskshift} and \ref{thm: diversitysufficient}. We begin by providing a more precise statement of Proposition \ref{thm: oodgennecessary} below.

\begin{proposition}\label{prop: simultaneousdiagonalizable}
    Let $P_{\A}$ and $P_{\A}'$ denote the pre-training and downstream task distributions, and suppose that the matrices in $\textrm{supp}(P_{\A})$ are simultaneously diagonalizable for a common orthogonal  matrix $U$. Suppose additionally that there exist matrices $A_1, A_2 \in \textrm{supp}(P_{\A})$ and $A_1' A_2' \in \textrm{supp}(P_{\A}')$ such that $A_1 A_2^{-1}$ and $A_1' (A_2')^{-1}$ have no repeated eigenvalues. 
    \begin{enumerate}
        \item If $\textrm{supp}(P_{\A}')$ is also simultaneously diagonalizable with respect to $U$, then $P_{\A}$ is diverse relative to $P_{\A}'$.
        \item If there exist matrices $A_3', A_4' \in \textrm{supp}(P_{\A}')$ such that $A_3' (A_4')^{-1}$ is not diagonalizable with respect to $U$, then $P_{\A}$ is not diverse relative to $P_{\A}'$.
    \end{enumerate}
\end{proposition}

In order to prove Proposition \ref{prop: simultaneousdiagonalizable}, we state and prove the following lemma. 

\begin{lemma}\label{lem: simuldiaglemma}
    Let $P_{\A}$ be a task distribution satisfying Assumption \ref{assumption:taskanddatadistr}. Suppose that the support of $P_{\A}$ is simultaneously diagonalizable with a common orthogonal diagonalizing matrix $U \in \R^{d \times d}$. Assume in addition that there exist $A_1, A_2 \in \textrm{supp}(P_{\A})$ such that $A_1 A_2^{-1}$ has distinct eigenvalues. Then $M_{\infty}(P_{\A}) = \Theta_{U,\Sigma},$ where $$\Theta_{U,\Sigma} := \left\{(P,\Sigma^{-1}P^{-1}): P = UDU^T, \; D = \textrm{diag}(\lambda_1, \dots, \lambda_d) \right\}.$$
\end{lemma}

\begin{proof}
    By Proposition \ref{minimizercharacterization}, a parameter $(P,Q)$ belongs to $\mathcal{M}_{\infty}(P_{\A})$ if and only if $P$ commutes with all products of the form $\{A_i A_j^{-1}: A_i, A_j \in \textrm{supp}(P_{\A})\}$, in which case $Q$ is defined by $Q = \Sigma^{-1} A_0 P^{-1} A_0^{-1}$ for any $A_0 \in \textrm{supp}(P_{\A}).$ Let $A_1, A_2 \in \textrm{supp}(P_{\A})$ be as defined in the statement of the lemma. Since $P$ and $A_1 A_2^{-1}$ are commuting diagonalizing matrices and $A_1 A_2^{-1}$ has no repeated eigenvalues (\cite{strang2022introduction}), they must be simultaneously diagonalizable. This implies that $P$ is diagonal in the basis $U$, and hence $Q$ is given by $Q = \Sigma^{-1} A_0 P^{-1} A_0^{-1} = \Sigma^{-1} P^{-1}.$
\end{proof}

\begin{proof}[Proof of Proposition \ref{prop: simultaneousdiagonalizable}]
    For 1), if the support of $P_{\A}'$ is also simultaneously diagonalizable with respect to $U$, then Lemma \ref{lem: simuldiaglemma} implies that $\mathcal{M}_{\infty}(P_{\A}) = \mathcal{M}_{\infty}(P_{\A}') = \Theta_{U,\Sigma}$, where $\Theta_{U,\Sigma}$, where $\Theta_{U,\Sigma}$ is as defined in the statement of Lemma \ref{lem: simuldiaglemma}. This proves that if the support of $P_{\A}'$ is also simultaneously diagonalizable with respect to $U$, then $P_{\A}$ is diverse.
    
    For 2), we must find a minimizer of $\mathcal{R}_{\infty}$ which is not a minimizer of $\mathcal{R}_{\infty}'$. Consider the parameter $\theta = (P, \Sigma^{-1} P^{-1}),$ where $P = U D U^T$ for $D$ an invertible diagonal matrix with no repeated entries. By Lemma \ref{lem: simuldiaglemma}, $\theta$ is a minimizer of $\mathcal{R}_{\infty}$. Let $A_3', A_4' \in \textrm{supp}(P_{\A}')$ be such that $A_3' (A_4')^{-1}$ is not diagonalizable with respect to $U$. Since $A_3' (A_4')^{-1}$ and $P$ are not simultaneously diagonalizable and $P$ has no repeated eigenvalues (\cite{strang2022introduction}), $P$ does not commute with $A_3' (A_4')^{-1}$. By Proposition \ref{minimizercharacterization}, $\theta$ is therefore not a minimizer of $\mathcal{R}_{\infty}'$, completing the proof.
\end{proof}

It only remains to prove Theorem \ref{thm: covariateshift}. In the prior results, we have considered the covariance matrix $\Sigma$ of the normally-distributed covariates to be fixed. To study shifts in the covariate distribution, it will be convenient to use the notation $\mathcal{R}_m^{\Sigma}$ for the $m$-sample population risk where the covariates $x_1, \dots, x_{n+1}$ are normally distributed with covariance $\Sigma.$ We begin by stating a more formal version of Theorem \ref{thm: covariateshift} where the constants are more explicit. 

\begin{theorem}\label{covarianceshiftformal}
     Let $\Sigma =  W \Lambda W^T$ and $\tilde{\Sigma} = \tilde{W} \tilde{\Lambda} \tilde{W}^T$ be two covariance matrices, let $(\widehat{P}, \widehat{Q})$ be minimizers of the empirical risk when the in-context examples follow the distribution $\mathcal{N}(0,\Sigma)$ and take $M > 0$ such that $\max\Big( \|\widehat{P}\|_{F}, \|\widehat{Q}\|_{F} \Big) \leq M.$ Then
    \begin{align*}  \mathcal{R}_m^{\tilde{\Sigma}}(\widehat{P},\widehat{Q}) &\lesssim \mathcal{R}_m^{\Sigma}(\widehat{P},\widehat{Q}) + c_A^2 M^4 \max(\|\Sigma\|_{\textrm{op}}, \|\tilde{\Sigma}\|_{\textrm{op}})^2 \|\Sigma - \tilde{\Sigma}\|_{\textrm{op}} \\ &+ \frac{1}{m} \cdot c_A^2 M^4 \max(\|\Sigma\|_{\textrm{op}}, \|\tilde{\Sigma}\|_{\textrm{op}})^2 \trace(\tilde{\Sigma}) \Big( \|\Sigma - \tilde{\Sigma}\|_{\textrm{op}} + \|\Lambda - \tilde{\Lambda}\|_1 + \|W - \tilde{W}\|_{\textrm{op}} \Big).
    \end{align*}
\end{theorem}
Theorem \ref{thm: covariateshift} then follows from Theorem \ref{covarianceshiftformal} by bounding $\|\Lambda - \tilde{\Lambda}\|_1 \lesssim \|\Sigma - \tilde{\Sigma}\|_{\textrm{op}}$, merging the term
$$ \frac{1}{m} \cdot c_A^2 M^4 \max(\|\Sigma\|_{\textrm{op}}, \|\tilde{\Sigma}\|_{\textrm{op}})^2 \trace(\tilde{\Sigma}) \Big( \|\Sigma - \tilde{\Sigma}\|_{\textrm{op}} + \|\Lambda - \tilde{\Lambda}\|_1 \Big)
$$
into the second term, and omitting the constant factors.

\begin{proof}[Proof of Theorem \ref{covarianceshiftformal}]
        By the triangle inequality, we have
    \begin{equation}\label{trineq} \mathcal{R}_m^{\tilde{\Sigma}}(\widehat{P},\widehat{Q}) \leq \mathcal{R}_m^{\Sigma}(\widehat{P},\widehat{Q}) + \sup_{\|P\|_{\textrm{op}}, \|Q\|_{\textrm{op}} \leq M} \Big| \mathcal{R}_m^{\tilde{\Sigma}}(P,Q) - \mathcal{R}_m^{\Sigma}(P,Q) \Big|.
    \end{equation}
    It therefore suffices to bound the second term. From the proof of Proposition \ref{lem: popriskexpression}, we know that
    \begin{align}\label{traceequation} \mathcal{R}_m^{\Sigma}(P,Q) &= \E_{A}\Big[\frac{m+1}{m}  \trace(PA^{-1} \Sigma Q \Sigma Q^T \Sigma A^{-1}P^T +\frac{\trace_{\Sigma}(Q\Sigma Q^T)}{m} \trace(PA^{-1} \Sigma A^{-1} P^T )\Big] \\
    &+ \E_{A} \Big[ \trace(A^{-1} \Sigma A^{-1}) - \trace(PA^{-1} \Sigma Q \Sigma A^{-1}) - \trace(A^{-1} \Sigma Q^T \Sigma A^{-1} P^T) \Big].
    \end{align}

    Similarly, we have
    \begin{align}\label{traceequation2} \mathcal{R}_m^{\tilde{\Sigma}}(P,Q) &= \E_{A}\Big[\frac{m+1}{m}  \trace(PA^{-1} \tilde{\Sigma} Q \tilde{\Sigma} Q^T \tilde{\Sigma} A^{-1}P^T +\frac{\trace_{\tilde{\Sigma}}(Q\tilde{\Sigma} Q^T)}{m} \trace(PA^{-1} \tilde{\Sigma} A^{-1} P^T )\Big] \\
    &+ \E_{A} \Big[ \trace(A^{-1} \tilde{\Sigma} A^{-1}) - \trace(PA^{-1} \tilde{\Sigma} Q \tilde{\Sigma} A^{-1}) - \trace(A^{-1} \tilde{\Sigma} Q^T \tilde{\Sigma} A^{-1} P^T) \Big].
    \end{align}
    We seek to bound the difference $\Big| \mathcal{R}_m^{\Sigma}(\theta) - \mathcal{R}_m^{\tilde{\Sigma}}(\theta) \Big|$ by bounding the respective differences of each term appearing in the expressions for $\mathcal{R}_m^{\Sigma}$ and $\mathcal{R}_m^{\tilde{\Sigma}}.$ By a simple applications of H\"older's inequality and the triangle inequality, we see that
    \begin{align*}
        \E_{A}\trace(PA^{-1}( \Sigma Q \Sigma - \tilde{\Sigma} Q \tilde{\Sigma}) A^{-1}) &\leq \E_{A}\|A^{-1} P A^{-1}\|_F \| \Sigma Q \Sigma - \tilde{\Sigma} Q \tilde{\Sigma}\|_F \\
        &\leq c_A^2 \|P\|_F \Big( \|(\Sigma - \tilde{\Sigma})Q \Sigma\|_F + \|\tilde{\Sigma} Q ( \Sigma - \tilde{\Sigma})\|_F\Big)  \\
        &\leq c_A^2 \|P\|_F \Big(\|Q \Sigma\|_F + \|\tilde{\Sigma} Q\|_F \Big) \|\Sigma-\tilde{\Sigma}\|_{\textrm{op}} \\
        &\leq 2 c_A^2 \|P\|_F \|Q_f \max(\|\Sigma\|_{\textrm{op}}, \|\tilde{\Sigma}\|_{\textrm{op}}) \|\Sigma - \tilde{\Sigma}\|_{\textrm{op}} \\
        &= 2 c_A^2 M^2  \max(\|\Sigma\|_{\textrm{op}}, \|\tilde{\Sigma}\|_{\textrm{op}}) \|\Sigma - \tilde{\Sigma}\|_{\textrm{op}}.
    \end{align*}
    Analogous arguments can be used to prove the bounds
    $$ \E_{A} \trace(A^{-1} (\Sigma Q^T \Sigma - \tilde{\Sigma} Q^T \tilde{\Sigma}) A^{-1} P^T) \leq 2 c_A^2 M^2  \max(\|\Sigma\|_{\textrm{op}}, \|\tilde{\Sigma}\|_{\textrm{op}}) \|\Sigma - \tilde{\Sigma}\|_{\textrm{op}},
    $$
    $$ \E_{A} \trace(A^{-1} (\Sigma - \tilde{\Sigma}) A^{-1}) \leq c_A^2 \|\Sigma - \tilde{\Sigma}\|_{\textrm{op}}
    $$
    and
    \begin{align*}
        \E_{A} \trace(PA^{-1} (\Sigma Q \Sigma Q^T \Sigma - \tilde{\Sigma} Q  \tilde{\Sigma} Q^T  \tilde{\Sigma}) A^{-1} P^T) \leq c_A^2 M^4 \max(\|\Sigma\|_{\textrm{op}}, \|\tilde{\Sigma}\|_{\textrm{op}})^2 \|\Sigma - \tilde{\Sigma}\|_{\textrm{op}}.
    \end{align*}
    Notice that the term above dominates each of the preceding three terms. For the final term, we have
    \begin{align*}
        &\trace_{\Sigma}(Q \Sigma Q^T) \trace(P A^{-1} \Sigma A^{-1} P^T) - \trace_{\tilde{\Sigma}}(Q \tilde{\Sigma} Q^T) \trace(P A^{-1} \tilde{\Sigma} A^{-1} P^T) \\ &\leq \Big|\trace_{\Sigma}(Q \Sigma Q^T) - \trace_{\tilde{\Sigma}}(Q \tilde{\Sigma} Q^T) \Big| \Big|\trace(P A^{-1} \Sigma A^{-1} P^T) \Big| \\
        & \qquad + \Big| \trace_{\tilde{\Sigma}}(Q \tilde{\Sigma} Q^T) \Big| \Big|\trace(P A^{-1} (\Sigma - \tilde{\Sigma}) A^{-1} P^T) \Big|.
    \end{align*}
    By Lemma \ref{normboundxi} and Holder's inequality, the second term satisfies
    \begin{align*}
        &\Big| \trace_{\tilde{\Sigma}}(Q \tilde{\Sigma} Q^T) \Big| \Big|\trace(P A^{-1} (\Sigma - \tilde{\Sigma}) A^{-1} P^T) \Big| \leq c_A^2 M^4 \|\tilde{\Sigma}\|_{\textrm{op}} \trace(\tilde{\Sigma}) \cdot \|\Sigma - \tilde{\Sigma}\|_{\textrm{op}}.
    \end{align*}
    Similarly, using Lemma \ref{technicallemma2}, the first term satisfies
    \begin{align*}
       & \Big|\trace_{\Sigma}(Q \Sigma Q^T) - \trace_{\tilde{\Sigma}}(Q \tilde{\Sigma} Q^T) \Big| \Big|\trace(P A^{-1} \Sigma A^{-1} P^T) \Big| \\ &\leq c_A^2 M^4 \|\Sigma\|_{\textrm{op}} \Big( \trace(\tilde{\Sigma}) \|\Sigma - \tilde{\Sigma}\|_{\textrm{op}} + \|\Sigma\|_{\textrm{op}} \Big(\|\Lambda - \tilde{\Lambda}\|_1 + 2 \trace(\tilde{\Sigma}) \|W - \tilde{W}\|_{\textrm{op}} \Big) \Big)
    \end{align*}
    Combining the estimates for each individual term and taking the supremum over the all $P,Q$ with Frobenius norm bounded by $M$ yields the final bound
    \begin{align*}
        \mathcal{R}_m^{\tilde{\Sigma}}(\widehat{P},\widehat{Q}) &\lesssim \mathcal{R}_m^{\Sigma}(\widehat{P},\widehat{Q}) + c_A^2 M^4 \max(\|\Sigma\|_{\textrm{op}}, \|\tilde{\Sigma}\|_{\textrm{op}})^2 \|\Sigma - \tilde{\Sigma}\|_{\textrm{op}} \\ &+ \frac{1}{m} \cdot c_A^2 M^4 \max(\|\Sigma\|_{\textrm{op}}, \|\tilde{\Sigma}\|_{\textrm{op}})^2 \trace(\tilde{\Sigma}) \Big( \|\Sigma - \tilde{\Sigma}\|_{\textrm{op}} + \|\Lambda - \tilde{\Lambda}\|_1 + \|W - \tilde{W}\|_{\textrm{op}} \Big).
    \end{align*}
\end{proof}

\section{Proofs for Section \ref{sec:icon}}

\begin{proof}[Proof of Corollary \ref{cor: generalizationellipticpde}]
    We first bound the quantities $\epsilon_{\mathcal{E},\mathcal{D}}$ and $C_{\mathcal{D}}$ when $\mathcal{T}$ is the set of solution operators corresponding to 1D uniformly-elliptic PDEs, the embedding $\mathcal{E}$ is the projection operator onto the first $d$ elements of a $P^1$ finite element basis, and the norm $\| \cdot \|_{\my}$ is the $H^1$-norm. To bound the discretization error constant $\epsilon_{\mathcal{E},\mathcal{D}}$, let $u_d$ denote the solution to the elliptic PDE \eqref{eqn: ellipticpde} when the coefficients are projected onto the span of the first $d$ finite elements. Then by Theorem 3.16 in \cite{ern2004theory}, the estimate
    $$ \|u-u_d\|_{H^1} \lesssim \frac{1}{d} \|u\|_{H^2}
    $$
    holds. Moreover, by elliptic regularity, we have
    $$ \|u\|_{H^2} \lesssim \|f\|_{L^2}.
    $$
    It follows that
    \begin{align*}
        \E\left[ \|u-u_d\|_{H^1}^2 \right] \lesssim \frac{1}{d^2} \E \left[ \|u\|_{H^2}^2\right] \lesssim \frac{1}{d^2} \E[\|f\|_{L^2}^2] = \frac{\trace(\Sigma_{\mathcal{X}})}{d^2},
    \end{align*}
    where $\Sigma_{\mathcal{X}}: L^2 \rightarrow L^2$ denotes the trace class covariance of the infinite-dimensional Gaussian covariate $f$. This proves that
    $$ \epsilon_{\mathcal{E},\mathcal{D}}^2 \lesssim \frac{\trace(\Sigma_{\mathcal{X}})}{d^2}.
    $$
    Next, we seek a bound on the constant $C_{\mathcal{D}}$ such that $\|g\|_{H^1} \leq C_{\mathcal{D}} \|g\|_{L^2}$ for all functions $g$ in the span of the first $d$ finite elements $\{\phi_k\}_{k=1}^{d}$. To this end, let $\Phi^\prime$ be the stiffness matrix defined by $\Phi^\prime_{ij} = \langle \phi^\prime_i, \phi^\prime_j \rangle_{L^2(\Omega)}$ and let $\Phi$ be the mass matrix defined by $\Phi_{ij} = \langle \phi_i, \phi_j\rangle_{L^2(\Omega)},$ both of which are symmetric positive-definite matrices. Then for any $g = \sum_{k=1}^{d} c_k \phi_k \in \textrm{span}\{\phi_k\}_{k=1}^{d},$ we have
        \begin{align*}
        \|g\|_{H^1(\Omega)}^2 &= \|g\|_{L^2(\Omega)}^2 + \Big\| \sum_{k=1}^{d} c_k \phi_k'(x) \Big\|_{L^2(\Omega)}^2 \\
        &=   c^T (\Phi + \Phi^\prime) c \\
        &= \tilde{c}(\mathbf{I}_d + \Phi^{-1/2}\Phi^\prime\Phi^{-1/2})\tilde{c}\\
        & \leq ( 1+ \lambda_{\max} (\Phi^{-1/2}\Phi^\prime\Phi^{-1/2})) \|\tilde{c}\|^2\\
        &  =(1+ \lambda_{\max} (\Phi^{-1/2}\Phi^\prime\Phi^{-1/2})) \|g\|^2_{L^2(\Omega)},
    \end{align*}
    where $\tilde{c} = \Phi c.$ It can be shown that for piecewise linear FEM on 1D, the stiffness and mass matrices satisfy $\lambda_{\max} (\Phi^{-1/2}\Phi^\prime\Phi^{-1/2}) \lesssim d^2$ (see e.g. equation (2.4) of \cite{boffi2010finite}). It follows that $C_{\mathcal{D}} \lesssim d^2.$ By Theorem \ref{thm: indomaingenoperator}, we have 
    \begin{align*}
         \E_{f_1, \dots, f_{m+1} \sim P_{\mx}, T \sim P_{\T}} \left[\left\|\md\left(y_{m+1}^{\widehat{\theta}}\right) - T(f_{m+1}) \right\|_{H^1(\Omega)}^2 \right] \lesssim \frac{1}{d^2} + d^2 \mathcal{R}_m(\widehat{\theta}). 
    \end{align*}
    Theorem \ref{thm:indomaingen} bounds the in-domain generalization error $\mathcal{R}_m(\widehat{\theta})$ by 
    $$ \mathcal{R}_m(\widehat{\theta}) \lesssim \frac{1}{m} + \frac{1}{n^2} + \frac{1}{\sqrt{N}}.
    $$
    However, the implicit constants hide a dependence on the discretization size $d.$ Looking at the proof of Theorem \ref{thm:indomaingen}, we see that
    \begin{align*}
         \mathcal{R}_m(\widehat{\theta}) \lesssim \frac{1}{m} + \frac{C_{\A}^4 \|\Sigma^{-1}\|_{\textrm{op}}}{n^2} + \frac{d^2 \|\Sigma^{-1}\|_{\op}^4}{\sqrt{N}},
    \end{align*}
    where the remaining implicit constants are bounded as $d \rightarrow \infty.$ The constant $C_{\mathcal{A}}$ is a bound on the FEM-discretization of the elliptic differential operator defining equation \eqref{eqn: ellipticpde}, which can be shown to be of order $d^2$ ($C_{\A}$ is bounded, e.g., by a constant multiple of the maximum eigenvalue of the discrete Laplace matrix). The matrix $\Sigma$ is the covariance of the source term $f \in L^2$ after being projected onto the span of $\{\phi_k\}_{k=1}^{d}.$ When $\Sigma_{\mathcal{X}} = \left( -\Delta + I \right)^{-\alpha}$ for some $\alpha>0$ which controls the smoothness of the source term, it follows from the inverse inequalities \cite[Lemma 12.1]{ern2004theory} that $\|\Sigma^{-1}\|_{\textrm{op}} \lesssim d^{2\alpha}$. Inserting these bounds into the previous estimate gives the final bound on the solution recovery error
    \begin{align*}
        \E \left[\left\|\md\left(y_{m+1}^{\widehat{\theta}}\right) - T(f_{m+1}) \right\|_{H^1(\Omega)}^2 \right] \lesssim \frac{1}{d^2} + \frac{d^2}{m} + \frac{d^{10+4\alpha}}{n^2} + \frac{d^{4+8\alpha}}{\sqrt{N}},
    \end{align*}
    where we have hidden all constants that do not depend on $d$. We conclude the proof.
\end{proof}

\section{Auxiliary lemmas}
We make frequent use of the following lemma to compute expectations of products of empirical covariance matrices.

\begin{lemma}\label{technicallemma}
    Let $\{x_1, \dots, x_n\} \subseteq \R^d$ be iid samples from $\mathcal{N}(0,\Sigma)$ and assume that $\Sigma = W \Lambda W^T$, where $\Lambda = \diag(\sigma_1^2, \dots, \sigma_d^2).$ Let $X_n = \frac{1}{n} \sum_{k=1}^{n} x_k x_k^T$ associated to $\{x_1, \dots, x_n\}$ and let $K \in \R^{d \times d}$ denote a deterministic symmetric matrix. Then
    $$ \E[X_n K X_n] = \frac{n+1}{n} \Sigma K \Sigma + \frac{\trace_{\Sigma}(K)}{n} \Sigma,
    $$
    where $\trace_{\Sigma}(K) := \sum_{\ell=1}^{d} \sigma_{\ell}^2 \langle K \varphi_{\ell},\varphi_{\ell} \rangle$ and $\varphi_{\ell} := W e_{\ell}$ denote the eigenvectors of $\Sigma.$
\end{lemma}
\begin{proof}
        Let us first consider the case that $W = \mathbf{I}_d$, so that the covariance is diagonal with entries $\sigma^2_1, \dots, \sigma^2_d$. Observe that
    \begin{align*}
        \E[(X_n K X_n)_{ij}] &= \E\Bigg[\sum_{\ell, \ell'=1}^{d} \frac{1}{n^2}\Bigg( \sum_{k \neq k'} \langle x_k, e_i \rangle \langle x_{k'}, e_j \rangle \langle x_k, e_{\ell} \rangle \langle x_{k'}, e_{\ell'}, \rangle K_{\ell, \ell'} \\ &+ \sum_{k=1}^{n} \langle e_i, x_k \rangle \langle e_j, x_k \rangle \langle e_{\ell}, x_k \rangle \langle e_{\ell'}, x_k \rangle K_{\ell, \ell'} \Bigg)\Bigg].
    \end{align*}
    When $i \neq j$, we compute that
    \begin{align*}
        \sum_{\ell, \ell'=1}^{d} \E\Big[\langle x_k, e_i \rangle \langle x_{k'}, e_j \rangle \langle x_k, e_{\ell} \rangle \langle x_{k'}, e_{\ell'}, \rangle K_{\ell, \ell'} \Big] = \sigma_i^2 \sigma_j^2 K_{i,j}
    \end{align*}
    and 
    $$ \sum_{\ell, \ell'=1}^{d} \E\Big[\langle x_k, e_i \rangle \langle x_{k}, e_j \rangle \langle x_k, e_{\ell} \rangle \langle x_{k}, e_{\ell'}, \rangle K_{\ell, \ell'} \Big] = 2 \sigma_i^2 \sigma_j^2 K_{i,j}.
    $$
For $i = j$, we have
$$ \sum_{\ell,\ell'=1}^{d} \E \Big[\langle x_k, e_i \rangle \langle y_{k'}, e_i \rangle \langle x_k, e_{\ell} \rangle \langle x_{k'}, e_{\ell'}, \rangle K_{\ell, \ell'} \Big] = \sigma_i^4 K_{i,i}
$$
and
\begin{align*}
    \sum_{\ell,\ell'=1}^{d} \E \Big[ \langle x_k, e_i \rangle^2 \langle x_k, e_{\ell} \rangle \langle x_{k}, e_{\ell'}, \rangle K_{\ell, \ell'} \Big] = 2\sigma_i^4 K_{i,i} + \sigma_i^2 \sum_{\ell=1}^{d} \sigma_{\ell}^2 K_{\ell,\ell}.
\end{align*}
Putting everything together, we have shown that
$$ \E(X_n K X_n)_{i,j}] = \frac{n+1}{n} \sigma_i^2 \sigma_j^2 K_{i,j} + \delta_{ij} \cdot \frac{\trace_{\Sigma}(K)}{n} \sigma_i^2.
$$
The result then follows since $(\Sigma K \Sigma)_{i,j} = \sigma_i^2 \sigma_j^2 K_{i,j}.$ For general covariance $\Sigma = W \Lambda W^T$, we have $X_n K X_n = W (Z_n W^T K W Z_n) W^T$, where $Z_n$ is the empirical covariance matrix associated to $\{W^T x_1, \dots, W^T x_n\}$. Noting that $W^T x \sim N(0, \Lambda)$ for $x \sim N(0, \Sigma)$, we can apply the above result to $W^T K W:$
\begin{align*}
    \E[X_n K X_n] &= W \E[Z_n (W^T K W) Z_n] W^T \\
    &= W \Big( \frac{n+1}{n} \Lambda W^T K W \Lambda + \frac{\trace_{\Sigma}(K)}{n} \Lambda \Big) W^T \\
    &= \frac{n+1}{n} \Sigma K \Sigma + \frac{\trace_{\Sigma}(K)}{n} \Sigma.
\end{align*}
\end{proof}

We quickly put Lemma \ref{technicallemma} to work to give a tractable expression for the population risk.
\begin{lemma}\label{lem: popriskexpression}
    For $\theta = (P,Q)$, we have
    \begin{align*} \mathcal{R}_n(\theta) &:= \E_{A, X_n}[\|(PA^{-1}X_n Q - A^{-1})\Sigma^{1/2}\|_F^2] = \E_{A}[\|(PA^{-1}\Sigma Q - A^{-1})\Sigma^{1/2}\|_F^2] \\ &+ \frac{1}{n}\E_{A} \Big[\trace(PA^{-1} \Sigma Q \Sigma Q^T \Sigma A^{-1} P^T) + \trace_{\Sigma}(Q \Sigma Q^T) \trace(P A^{-1} \Sigma A^{-1} P^T) \Big].
    \end{align*}
\end{lemma}
\begin{proof}
    This follows from a direct computation of the expectation with respect to $X_n:$
    \begin{align*}
        &\E_{A, X_n}[\|(PA^{-1}X_n Q - A^{-1})\Sigma^{1/2}\|_F^2] = \E_{A,X_n}[\trace((PA^{-1}X_n Q - A^{-1})\Sigma(Q^T X_n A^{-1} P^T - A^{-1}))] \\
        &= \E_{A,X_n}[\trace(A^{-1}\Sigma A^{-1} + PA^{-1}X_nQ\Sigma Q^T X_n A^{-1} P^T -PA^{-1}X_n Q \Sigma A^{-1} - A^{-1} \Sigma Q^T X_n A^{-1} P^T)] \\
        &= \E_{A}[\trace(A^{-1} \Sigma A^{-1} - PA^{-1} \Sigma Q \Sigma A^{-1} - A^{-1} \Sigma Q^T \Sigma A^{-1} P^T] \\
        & \qquad + \E_{A,X_n}[\trace(PA^{-1}X_nQ\Sigma Q^T X_n A^{-1} P^T)] \\
        &= \E_{A}[\trace(A^{-1} \Sigma A^{-1} - PA^{-1} \Sigma Q \Sigma A^{-1} - A^{-1} \Sigma Q^T \Sigma A^{-1} P^T] \\ &+ \frac{n+1}{n} \E_{A}[\trace(PA^{-1} \Sigma Q \Sigma Q^T \Sigma A^{-1} P^T)] + \frac{1}{n}\E_{A}[\trace_{\Sigma}(Q\Sigma Q^T) \trace(PA^{-1}\Sigma A^{-1} P^T)] \\
        &= \E_{A}[\|(PA^{-1}\Sigma Q - A^{-1})\Sigma^{1/2}\|_F^2] \\
        & \qquad + \frac{1}{n}\E_{A} \Big[\trace(PA^{-1} \Sigma Q \Sigma Q^T \Sigma A^{-1} P^T) + \trace_{\Sigma}(Q \Sigma Q^T) \trace(P A^{-1} \Sigma A^{-1} P^T) \Big],
    \end{align*}
    where we used Lemma \ref{technicallemma} to compute the expectation over $X_n$ in the second-to-last line.
\end{proof}

It will also be useful to derive a simpler expression for the population risk $\mathcal{R}_m(\theta)$ when $\theta$ belongs to the set $\Theta_{\Sigma} = \{(c \mathbf{I}_d, c^{-1} \Sigma^{-1}): c \in \R \backslash \{0\} \}.$

\begin{lemma}\label{specialpopriskexpression}
    Let $P = c \mathbf{I}_d$, $Q = c^{-1} \Sigma^{-1}$ for $c \in \R \backslash \{0\}.$ Then $$\mathcal{R}_m(\theta) = \frac{d+1}{n} \E_A \Big[ \trace \Big(A^{-1} \Sigma A^{-1} \Big) \Big].$$
\end{lemma}
\begin{proof}
    Using Lemma \ref{technicallemma} to compute the expectations defining $\mathcal{R}_m,$ we have
    \begin{align*} \mathcal{R}_m(\theta) &= \E_{A}[\trace(A^{-1} \Sigma A^{-1} - PA^{-1} \Sigma Q \Sigma A^{-1} - A^{-1} \Sigma Q^T \Sigma A^{-1} P^T] \\ &+ \frac{n+1}{n} \E_{A}[\trace(PA^{-1} \Sigma Q \Sigma Q^T \Sigma A^{-1} P^T)] + \frac{1}{n}\E_{A}[\trace_{\Sigma}(Q\Sigma Q^T) \trace(PA^{-1}\Sigma A^{-1} P^T)].
    \end{align*}
    Since $P = c \mathbf{I}_d$ and $Q = c^{-1} \Sigma^{-1}$, we have that $P A^{-1} \Sigma Q \Sigma A^{-1}$, $A^{-1} \Sigma Q^T \Sigma A^{-1} P^T$, and $P A^{-1} \Sigma Q \Sigma Q^T \Sigma A^{-1} P^T$ are all equal to $A^{-1} \Sigma A^{-1},$ and $$\E_A \trace_{\Sigma}(Q \Sigma Q^T) \trace(P A^{-1} \Sigma A^{-1} P^T) = \E_A\trace_{\Sigma}(\Sigma^{-1}) \trace(A^{-1} \Sigma A^{-1}).$$ Therefore, after some algebra, the population risk simplifies to
    $$ \mathcal{R}_m(\theta) = \frac{1 + \trace_{\Sigma}(\Sigma^{-1})}{n} \E_A \Big[ \trace \Big(A^{-1} \Sigma A^{-1} \Big) \Big].
    $$
    Noting that $\trace_{\Sigma}(\Sigma^{-1}) = d$, we conclude the expression for $\mathcal{R}_m(\theta)$ as stated in the lemma.
\end{proof}

We quote the following result from Theorem 2.1 of \cite{rudelson2013hanson}.

\begin{lemma}\label{gaussianconc}[Gaussian concentration bound]
    Let $y \sim \mathcal{N}(0,\Sigma)$. Then
    $$ \mathbb{P}\left\{\|y\| \geq \sqrt{\trace(\Sigma)} + t \right\} \leq 2 \exp \Big(-\frac{t^2}{C \|\Sigma\|_{\textrm{op}}} \Big),
    $$
    where $C > 0$ is a constant independent of $\Sigma$ and $d.$
\end{lemma}

We use the following result to control the error between $Q_n$ and $\Sigma^{-1}$.

\begin{lemma}\label{perturbationlemma}
    \item Let $Q_n = B\Big(\frac{n+1}{n} B \Sigma + \frac{\trace_{\Sigma}(B)}{n} \Sigma \Big)^{-1}$. Assume that $n$ satisfies
    $$ \frac{\|\Sigma^{-1}\|_{\textrm{op}}\Big\|\Sigma\Big( \mathbf{I}_d + \trace_{\Sigma}(B) B^{-1} \Big)\Big\|_{\textrm{op}}}{n} \leq \frac{1}{2}.
    $$
    Then we can write
    $$ Q_n = \Sigma^{-1} + \frac{1}{n} \mathcal{E}_1,
    $$
    where $\mathcal{E}_1$ satisfies
    $$ \|\mathcal{E}_1\| \lesssim \|\Sigma^{-1}\|_{\textrm{op}} \|\Sigma\|_{\textrm{op}} \Big(1 + \trace_{\Sigma}(B) \Big) C_A^2.
    $$
\end{lemma}
\begin{proof}
    Using some algebra, we find
    \begin{align*}
        Q_n &= B\Big(\frac{n+1}{n} B \Sigma + \frac{\trace_{\Sigma}(B)}{n} \Sigma \Big)^{-1} \\
        &= \Big(\frac{n+1}{n} \Sigma + \frac{\trace_{\Sigma}(B)}{n} \Sigma B^{-1} \Big)^{-1} \\
        &= \Big(\Sigma + \frac{1}{n} \Sigma \Big(\mathbf{Id} + \trace_{\Sigma}(B) B^{-1} \Big) \Big)^{-1}.
    \end{align*}
    By Lemma \ref{matrixinvperturbation}, we have
    $$ \| Q_n - \Sigma^{-1}\|_{\textrm{op}} \leq \|\Sigma^{-1}\|_{\textrm{op}} \cdot \frac{\epsilon^{\ast}}{1-\epsilon^{\ast}},
    $$
    where
    $$ \epsilon^{\ast} = \frac{\|\Sigma^{-1}\|_{\textrm{op}}\Big\|\Sigma\Big( \mathbf{I}_d + \trace_{\Sigma}(B) B^{-1} \Big)\Big\|_{\textrm{op}}}{n}.
    $$
    This gives the final bound
    $$ \| Q_n - \Sigma^{-1}\|_{\textrm{op}} \lesssim \frac{\|\Sigma^{-1}\|_{\textrm{op}}\Big\|\Sigma\Big( \mathbf{I}_d + \trace_{\Sigma}(B) B^{-1} \Big)\Big\|_{\textrm{op}}}{n} \leq \frac{\|\Sigma^{-1}\|_{\textrm{op}} \|\Sigma\|_{\textrm{op}} \Big(1 + \trace_{\Sigma}(B)\|B^{-1}\|_{\textrm{op}} \Big) }{n},
    $$
    Here, we used the bound $\frac{\epsilon}{1-\epsilon} \lesssim\epsilon$ which holds for $\epsilon$ sufficiently small; in particular, for $\epsilon \in (0,1/2),$ we have $\frac{\epsilon}{1-\epsilon} \leq 2 \epsilon.$
\end{proof}

The following result, used to bound the inverse of a perturbed matrix, is a standard application of matrix power series.
\begin{lemma}\label{matrixinvperturbation}
    Suppose that $A$ is an invertible $d \times d$ matrix and $D \in \R^{d \times d}$ satisfies $\|D\|_{\textrm{op}} \leq \frac{\epsilon}{\|A^{-1}\|_{\textrm{op}}}$ for some $\epsilon < 1.$ Then 
    $$ \|(A+D)^{-1} - A^{-1} \|_{\textrm{op}} \leq \|A^{-1}\|_{\textrm{op}} \cdot \frac{\epsilon}{1-\epsilon}.
    $$
\end{lemma}
\begin{proof}
    Note that $A+D = (\mathbf{I}_d + DA^{-1})A$. Under our assumption on $D$, we have $\|DA^{-1}\|_{\textrm{op}} \leq \|D\|_{\textrm{op}} \|A^{-1}\|_{\textrm{op}} < 1$, which implies the series expansion
    $$ (\mathbf{I}_d+DA^{-1})^{-1} = \sum_{k=0}^{\infty} (-DA^{-1})^k.
    $$
    It follows that
    \begin{align*}
        (A+D)^{-1} &= \Big( \Big(\mathbf{I}_d + DA^{-1} \Big) A \Big)^{-1} \\
        &= A^{-1} \Big(\mathbf{I}_d+DA^{-1} \Big)^{-1} \\
        &= A^{-1} \sum_{k=0}(-DA^{-1})^k.
    \end{align*} 
    In turn, this gives the bound
    \begin{align*}
        |(A+D)^{-1} - A^{-1} \|_{\textrm{op}} &= \Big\| A^{-1} \sum_{k=1}^{\infty} (-DA^{-1})^k 
        \Big\|_{\textrm{op}} \\
        &\leq \|A^{-1}\|_{\textrm{op}} \sum_{k=1}^{\infty} \|DA^{-1}\|_{\textrm{op}}^k \\
        &\leq  \|A^{-1}\|_{\textrm{op}} \sum_{k=1}^{\infty} \epsilon^k \\
        &= \|A^{-1}\|_{\textrm{op}} \frac{\epsilon}{1-\epsilon}.
    \end{align*}
\end{proof}

Recall that for a positive definite matrix $\Sigma = W \Lambda W^T$ and a symmetric matrix $K$, $$\trace_{\Sigma}(K) = \sum_{i=1}^{d} \sigma_i^2 \langle K \varphi_i, \varphi_i \rangle,$$ where $\sigma_1^2, \dots, \sigma_d^2$ are the eigenvalues of $\Sigma$ and $\varphi_i = W e_i$ are the eigenvectors of $\Sigma.$

\begin{lemma}\label{normboundxi}
    For any symmetric matrix $K,$ we have $$\trace_{\Sigma}(K) \leq \|K\|_{\textrm{op}} \trace(\Sigma).$$
\end{lemma}
\begin{proof}
    For each $1 \leq i \leq d,$ we have $\langle K \varphi_i, \varphi_i \rangle \leq \|K \varphi_i\| \|\varphi_i\| \leq \|K\|_{\textrm{op}}$. Therefore, 
    $$ \trace_{\Sigma}(K) = \sum_{i=1}^{d} \sigma_i^2 \langle K \varphi_i, \varphi_i \rangle \leq \|K\|_{\textrm{op}} \sum_{i=1}^{d} \sigma_i^2 = \|K\|_{\textrm{op}} \trace(\Sigma).$$
\end{proof}

In order to prove Theorem \ref{thm: covariateshift}, we also need the following stability bound of $\trace_{\Sigma}(K)$ with respect to perturbations of both $\Sigma$ and $K.$

\begin{lemma}\label{technicallemma2}
    Let $\Sigma = W \Lambda W^T$ and $\tilde{\Sigma} = \tilde{W} \tilde{\Lambda} \tilde{W}^T$ be two symmetric positive definite matrices and $K, \tilde{K}$ two symmetric matrices, let $\{\sigma_i^2\}_{i=1}^{d}$ and $\{\tilde{\sigma}_i^2\}_{i=1}^{d}$ be the respective eigenvalues of $\Sigma$ and $\tilde{\Sigma}$ and let $\{\varphi_i\}_{i=1}^{d}$ and $\{\tilde{\varphi}_i\}_{i=1}^{d}$ be the respective eigenvectors. Then
    $$ \Big| \trace_{\Sigma}(K) - \trace_{\tilde{\Sigma}} \tilde{K} \Big| \leq \trace(\tilde{\Sigma}) \|K-\tilde{K}\|_{\textrm{op}} + \|K\|_{\textrm{op}} \Big( \|\Lambda - \tilde{\Lambda}\|_1 + 2 \trace(\tilde{\Sigma}) \|W - \tilde{W}\|_{\textrm{op}} \Big).
    $$
\end{lemma}
\begin{proof}
    We have \begin{equation}\label{technicallemma2eq1} \trace_{\Sigma}(K) - \trace_{\tilde{\Sigma}}(\tilde{K}) \leq \Big|\trace_{\Sigma}(K) - \trace_{\tilde{\Sigma}}(K) \Big| + \Big| \trace_{\tilde{\Sigma}}(K - \tilde{K}) \Big|.\end{equation} 
The second term in \eqref{technicallemma2eq1} can be bounded by an application of Lemma \ref{normboundxi}, which yields $$\Big| \trace_{\tilde{\Sigma}}(K - \tilde{K}) \Big| \leq \trace(\tilde{\Sigma}) \|K- \tilde{K}\|_{\textrm{op}}.$$ To bound the first term in \eqref{technicallemma2eq1}, we first use the estimate
\begin{align*}
    \Big|\trace_{\Sigma}(K) - \trace_{\tilde{\Sigma}}(K) \Big| &\leq \Big|\sum_{i=1}^{d} \Big(\sigma_i^2 - \tilde{\sigma}_i^2 \Big) \langle K \varphi_i, \varphi_i \rangle  \Big| + \Big| \sum_{i=1}^{d} \tilde{\sigma}_i^2 \Big( \langle K (\varphi_i - \tilde{\varphi}_i), \varphi_i \rangle + \langle K \tilde{\varphi}_i, \varphi_i - \tilde{\varphi_i} \rangle\Big) \Big|.
\end{align*}
The first term above can be bounded by
\begin{equation}\label{technicallemma2eq2} \Big|\sum_{i=1}^{d} \Big(\sigma_i^2 - \tilde{\sigma}_i^2 \Big) \langle K \varphi_i, \varphi_i \rangle  \Big| \leq \|K\|_{\textrm{op}} \cdot \sum_{i=1}^{d} \Big| \sigma_i^2 - \tilde{\sigma}_i^2 \Big| = \|K\|_{\textrm{op}} \cdot \|\Lambda - \tilde{\Lambda}\|_{1}.
\end{equation}
To bound the second term in \eqref{technicallemma2eq2}, note that for any $1 \leq i \leq d,$ we have $$\langle K(\varphi_i - \tilde{\varphi}_i, \varphi_i \rangle \leq \|K\|_{\textrm{op}} \|\varphi_i-\varphi_i\| \leq \|K\|_{\textrm{op}} \|W-\tilde{W}\|_{\textrm{op}},$$ and similarly $\langle K \tilde{\varphi}, \varphi-\tilde{\varphi} \rangle \leq \|K\|_{\textrm{op}} \|W - \tilde{W}\|_{\textrm{op}}.$ It therefore holds that
\begin{align*} 
    \Big| \sum_{i=1}^{d} \tilde{\sigma}_i^2 \Big( \langle K (\varphi_i - \tilde{\varphi}_i), \varphi_i \rangle + \langle K \tilde{\varphi}_i, \varphi_i - \tilde{\varphi_i} \rangle\Big) \Big| &\leq 2 \|K\|_{\textrm{op}} \trace(\tilde{\Sigma}) \|W - \tilde{W}\|_{\textrm{op}}.
\end{align*}
Combining all terms yields the final estimate
\begin{align*}
    \Big| \trace_{\Sigma}(K) - \trace_{\tilde{\Sigma}} \tilde{K} \Big| \leq \trace(\tilde{\Sigma}) \|K-\tilde{K}\|_{\textrm{op}} + \|K\|_{\textrm{op}} \Big( \|\Lambda - \tilde{\Lambda}\|_1 + 2 \trace(\tilde{\Sigma}) \|W - \tilde{W}\|_{\textrm{op}} \Big).
\end{align*}
\end{proof}

The following lemma is an adaptation of Wald's consistency theorem of M-estimators \cite[Theorem 5.14]{van2000asymptotic}. We use it to prove the convergence in probability of empirical risk minimizers to population risk minimizers.

\begin{lemma}\label{wald}
    Let $\theta \in \R^m$, $x \in \R^d$, and suppose $\ell(\cdot,\cdot): \R^{d} \times \R^m \rightarrow [0,\infty)$ is lower semicontinuous in $\theta$. Let $m_0 = \textrm{min}_{\theta} \E[\ell(x,\theta)]$ for some fixed distribution on $x$, and let $\Theta_0 = \textrm{argmin}_{\theta} \E[\ell(x,\theta)]$. Let $\{\theta_N\}_{N \in \mathbb{N}}$ be a collection of estimators such that $\sup_N \|\theta_N\| < \infty$ and $$m_0 - \E_N[\ell(x,\theta_0)] = o_P(1)$$ Then $\textrm{dist}(\theta_N,\Theta_0) \inp 0.$
\end{lemma}

\begin{theorem}\label{convergenceofminimizers}
  For any sequence $\{\widehat{\theta}_{n,N}\}_{n,N \in \mathbb{N}}$ of minimizers of the empirical risk $\mathcal{R}_{n,N}$ with $\sup_N \|\widehat{\theta}_{n,N}\| < \infty$ for all $n$, we have
    $$ \lim_{n \to \infty} \lim_{N to \infty} \textrm{dist}(\widehat{\theta}_{n,N}, \mathcal{M}_{\infty}) = 0, \; \textrm{in probability}.
    $$
\end{theorem}
\begin{proof}
    For each fixed $n \in \mathbb{N}.$ we can apply Lemma \ref{wald} to the empirical risk minimizer $\widehat{\theta}_{n,N}$. In this context, the condition of the lemma amounts to the condition that $\mathcal{R}_n(\theta_{\ast}) - \mathcal{R}_{n,N}(\widehat{\theta}_{n,N}) = o_P(1),$ for any $\theta_{\ast} \in \textrm{argmin}_{\theta} \mathcal{R}_n,$ which is satisfied since
$$ \mathcal{R}_n(\theta_{\ast}) - \mathcal{R}_{n,N}(\widehat{\theta}_{n,N}) = \Big( \mathcal{R}_n(\theta_{\ast}) - \mathcal{R}_{n,N}(\theta_{\ast}) \Big) + \Big(\mathcal{R}_{n,N}(\theta_{\ast}) -  \mathcal{R}_{n,N}(\widehat{\theta}_{n,N})\Big).
$$
The first term tends to zero in probability by the law of large numbers, and the second term is non-negative by the minimality of $\widehat{\theta}_{n,N}.$ This proves that
$$ \lim_{N to \infty} \textrm{dist}(\widehat{\theta}_{n,N}, \mathcal{M}_{n}) = 0, \; \textrm{in probability},
$$
where $\mathcal{M}_n = \textrm{argmin}_{\theta} \mathcal{R}_n(\theta).$ Consequently, since $\mathcal{R}_n$ and $\mathcal{R}_{\infty}$ are polynomials in $\theta$ such that the coefficients of $\mathcal{R}_n$ converge to the coefficients of $\mathcal{R}_{\infty}$ as $n \rightarrow \infty,$ we have by the triangle inequality that
    $$\lim_{n \to \infty} \lim_{N \to \infty} \textrm{dist}(\widehat{\theta}_{n,N}, \mathcal{M}_{\infty}) = 0, \; \textrm{in probability}. $$
\end{proof}

\section{Additional numerical experiments on OOD generalization for ICL of PDEs} \label{sec:addnum} In this section, we provide some additional numerical results to validate the importance of task diversity for the in-context operator learning of linear elliptic PDE. We set $a(x) \equiv0.1$, and let source term $f(x)$ to be white noise. We train a linear transformer on task distribution $V = c\cdot  \bI_d, c \sim U[10,20]$, and then evaluate it on more diverse task distribution $V(x) \sim \mathbf{U}_d[a,b]$ across various $(a,b)$, and the results are presented in \ref{fig:PDE_not_diverse}. Similar to \ref{subsec: ood_not_diverse}, when training tasks lack of diversity the loss admits infinitely many minimizers, namely any diagonal matrix pair $(K, K^{-1})$ with $K \sim \mathbf{U}_d[10, 20]$, but only the minimizer $(\bI_d, \bI_d)$ generalizes to more diverse downstream tasks. In the left plot of \ref{fig:PDE_not_diverse}, a transformer initialized near $(\bI_d, \bI_d)$ generalizes to the downstream tasks, whereas the one initialized near $(K, K^{-1})$ (right plot) suffers degraded performance. These results are consistent with \ref{thm: oodgennecessary}, which emphasizes the importance of task diversity in achieving the OOD generalization.

\begin{figure}[H]
  \centering
  \includegraphics[width=\linewidth]{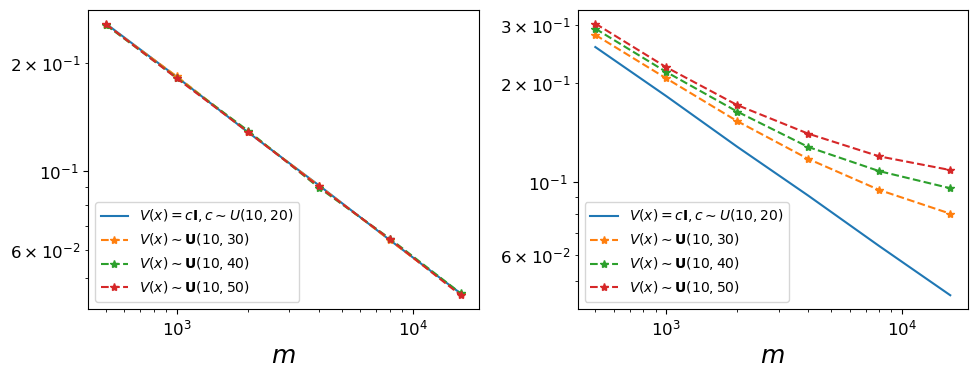}
  \caption{Diversity test with $d=32$, training prompt length $n=2000$, number of tasks $N=20000$. Both plots show the relative $H^1$ error with respect to varying testing prompt length $m$. The left plot corresponds to a transformer initialized near the minimizer $(P, Q) = (\mathbf{I}_d, \mathbf{I}_d)$, while the right plot uses initialization $(P, Q) = (K, K^{-1})$, where $K$ is a diagonal matrix with entries sampled from the uniform distribution $U(10, 20)$. }
  \label{fig:PDE_not_diverse}
\end{figure}


\end{document}